\documentclass{article}

 \usepackage[preprint]{neurips_2026}

\usepackage[utf8]{inputenc} 
\usepackage[T1]{fontenc}    
\usepackage{hyperref}       
\usepackage{url}            
\usepackage{booktabs}       
\usepackage{amsfonts}       
\usepackage{nicefrac}       
\usepackage{microtype}      
\usepackage{xcolor}         
\usepackage{graphicx}
 \usepackage{wrapfig}
\usepackage{amsmath}
\usepackage{amssymb}
\usepackage{mathtools}
\usepackage{amsthm}
\usepackage{pifont}
\newcommand{\cmark}{\ding{51}}%
\newcommand{\xmark}{\ding{55}}%

\usepackage{multirow}
\usepackage{hyperref}
\definecolor{nicergreen}{RGB}{0, 100, 0}
\definecolor{nicerblue}{RGB}{0, 51, 153}
\hypersetup{
    colorlinks=true,
    citecolor=nicerblue,
    linkcolor=nicerblue,
    urlcolor=nicerblue
}

\usepackage{enumitem}
\setlist[itemize]{
    leftmargin=*,        
    itemsep=3pt,         
    topsep=2pt,          
    parsep=0pt           
}

\theoremstyle{plain}
\newtheorem{theorem}{Theorem}[section]
\newtheorem{proposition}[theorem]{Proposition}
\newtheorem{lemma}[theorem]{Lemma}
\newtheorem{corollary}[theorem]{Corollary}
\theoremstyle{definition}
\newtheorem{definition}[theorem]{Definition}

\theoremstyle{remark}
\newtheorem{remark}[theorem]{Remark}

\newcommand{\method}{\texttt{STRAND}}

\title{From Persistence to Survival: Hypothesis Testing, Effect Sizes and Vectorisation for Topological Features}

\author{%
  Juliette Murris \quad Bernadette Stolz \quad Karsten Borgwardt 
  \vspace{0.2cm} \\
  Department of Machine Learning and Systems Biology, \\
  Max Planck Institute of Biochemistry, Martinsried, Germany 
  \vspace{0.2cm} \\
  Code: will be released upon publication.
}

\begin{document}

\maketitle

\begin{abstract}
    Persistence diagrams are common representations in topological data analysis, but they do not naturally live in a vector space, and the statistical tools developed for comparing them have largely evolved separately from those used for downstream prediction. We introduce \method~(Survival Topological Representation ANalysis of Diagrams), which treats (collections of) PDs as survival data: each topological feature with persistence value $p = d - b$ is a fully observed time-to-event, and the persistence survival function $S(t) = \mathbb{P}(p > t)$ is the central object for comparing diagrams. From this single representation we derive (i) a non-parametric two-sample test with calibrated Type~I error and high power from a small number of diagrams; (ii) interpretable effect sizes; and (iii) a 1-Wasserstein-stable feature vector for downstream machine learning. We validate calibration and power on synthetic manifolds with controlled topology, demonstrate competitive vectorisation across 14 graph and 3D point cloud benchmarks, and apply the method to study functional brain connectivity in fMRI/neuroscience data. To our knowledge, \method~is the first method to provide hypothesis testing and vectorisation for persistence diagrams from a single coherent and interpretable representation.
\end{abstract}

\section{Introduction}

The shape of data carries information that local statistics miss. A protein's function depends on how amino acids fold into loops and cavities; a brain's organisation manifests in the topology of its functional connections; a drug's efficacy depends on the shape of its molecular graph. Topological data analysis (TDA)~\citep{robins2000computational,ghrist2008barcodes,carlsson2009topology} provides mathematical tools to extract such shape-based descriptors. Persistent homology (PH), the most commonly used method in TDA, tracks topological features (connected components, loops, voids) as they appear and disappear across scales~\citep{edelsbrunner2002topological,zomorodian2004computing}. The output is a persistence diagram (PD), a multiset of birth-death pairs encoding the scales when each feature emerges and vanishes in the data.

Despite their descriptive power, PDs pose two challenges for practitioners~\citep{otter2017roadmap,chazal2021introduction}. First, given PDs from two conditions, topological difference characterisation remains descriptive rather than confirmatory without principled hypothesis tests~\citep{bubenik2007statistical,fasy2014confidence}. Second, PDs are multisets of varying cardinality, incompatible with machine learning algorithms that expect fixed-dimensional inputs~\citep{ali2023survey}. Existing work addresses these challenges separately, producing disconnected toolkits. Vectorisation methods~\citep{adamsPersistenceImagesStable2017,bubenik2017persistence,reininghaus2015stable,carriere2020perslay} enable classification but provide no statistical inference. Conversely, hypothesis testing methods~\citep{robinson2017hypothesis,islambekov2024vector,kwitt2015statistical,schrab2023mmd} yield $p$-values via permutation or kernel procedures but provide no effect sizes, no localisation of where groups differ, and no representation that connects to downstream prediction. A practitioner who wants both classification and statistical comparison must use two different methods operating on two different objects, with the features driving prediction different to the ones tested for significance, breaking the interpretive link between what is predicted and what is reported as significant.

We close this gap with \method~(Survival Topological Representation ANalysis of Diagrams) by mapping PDs to survival data (Figure~\ref{fig:figure1}). Each topological feature with persistence value $p_i = d_i - b_i$ is treated as a fully observed time-to-event observation. The obtained persistence survival function $S(t) = \mathbb{P}(p > t)$ enables comparison between groups, yielding a hypothesis test with interpretable effect sizes. Besides, its discretisation yields a stable feature vector for downstream machine learning. Our contributions are:
\begin{itemize}
    \item \textbf{Hypothesis testing for (collections of) persistence diagrams} (Section~\ref{sec:hypothesis_testing}). We introduce $S(t) = \mathbb{P}(p > t)$ as the central object for comparing collections of diagrams, and we use the log-rank test to compare $S_A$ to $S_B$ between two cohorts. The test maintains calibrated Type~I error on synthetic manifolds and achieves high power from as few as two diagrams per group.
    \item \textbf{Interpretable, dimension-stratified effect sizes} (Sections~\ref{sec:hypothesis_testing}--\ref{sec:abide}). Beyond the scalar $p$-values returned by existing PD inference, \method~reports hazard ratios with bootstrap confidence intervals, median persistence differences, and Kaplan-Meier curves that localise group differences across the persistence axis and across homological dimensions.
    \item \textbf{A stable vectorisation derived from the same object} (Section~\ref{sec:vectorisation}). Discretising $\hat S$ on a grid yields a fixed-dimensional, 1-Wasserstein-stable vector (Theorem~\ref{thm:stability}). It is competitive with state-of-the-art vectorisations across 14 graph and 3D point cloud benchmarks at 30--85$\times$ lower feature dimensionality than persistence images.
    \item \textbf{Application to autism connectivity} (Section~\ref{sec:abide}). On functional brain connectivity diagrams across three brain atlases, \method~jointly delivers per-dimension hypothesis tests with effect sizes, head-to-head against three established PD-inference baselines.
\end{itemize}





\begin{figure}[t]
  \begin{center}
    \centerline{\includegraphics[width=\linewidth]{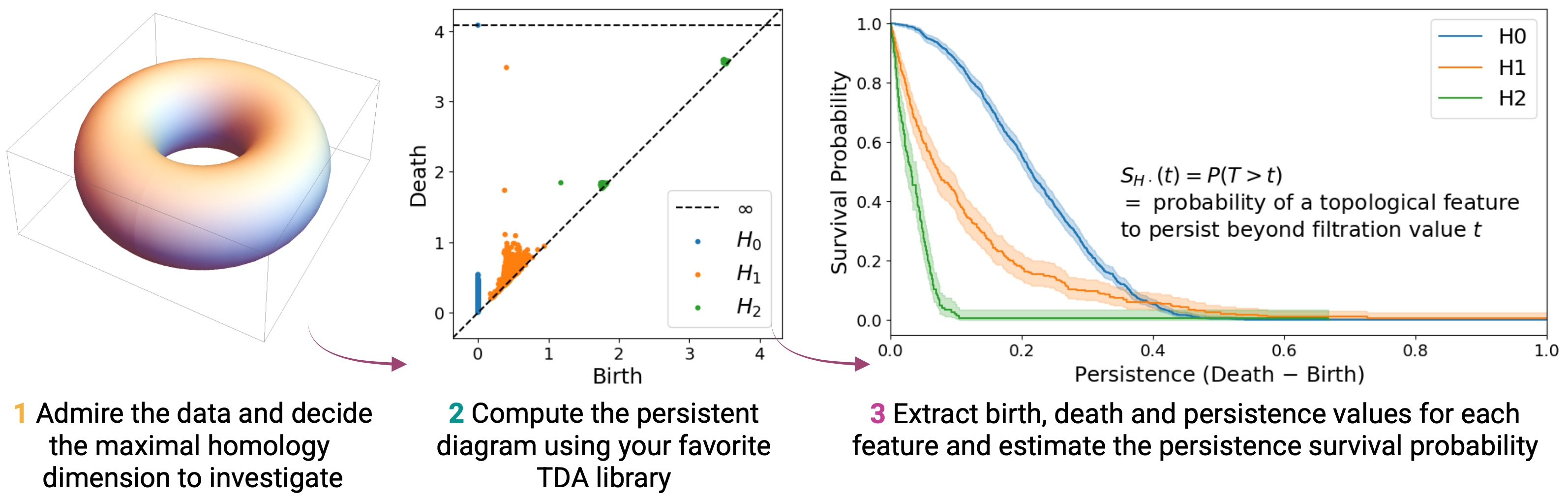}}
    \caption{From a point cloud (here, a torus in 1), persistent homology produces a persistence diagram with births and deaths in dimensions $H_0$, $H_1$, $H_2$ (2). Each feature is treated as a survival observation on the persistence axis $p = d - b$, and (3) \method~yields persistence survival probabilities.}
    \label{fig:figure1}
  \end{center}
\vspace{-0.7cm}
\end{figure}

\section{Related Work}\label{sec:bckgrd}

\paragraph{Persistent homology.}
Persistent homology (PH) tracks topological features across scales~\citep{edelsbrunner2002topological,zomorodian2004computing,carlsson2009topology}. Given a filtration $\mathcal{F} = \{K_\epsilon\}_{\epsilon \geq 0}$ of nested simplicial complexes (e.g., Vietoris-Rips on point clouds, degree-based on graphs), PH records when connected components ($H_0$), loops ($H_1$), and voids ($H_2$) appear and disappear. The output is a persistence diagram (PD) $\mathcal{D} = \{(b_i, d_i)\}_{i=1}^n$ with persistence $p_i = d_i - b_i$. PDs are stable under the Wasserstein and bottleneck distances~\citep{cohen2005stability,skraba2020wasserstein}.

\paragraph{Vectorisation.}
Vectorisation methods map PDs to fixed-dimensional vectors compatible with machine learning. Persistence Images~\citep{adamsPersistenceImagesStable2017} and Persistence Landscapes~\citep{bubenik2017persistence} produce stable representations via Gaussian smoothing and piecewise-linear functions respectively; Betti curves~\citep{giusti2015clique} and the broader Persistence Curves family~\citep{chung2022persistence} aggregate feature counts under a weight-function parametrisation (see Appendix~\ref{app:vs_betti} for a geometric comparison). Kernel methods~\citep{reininghaus2015stable,carriere2017sliced,le2018persistence} embed diagrams into reproducing kernel Hilbert space for kernel machines. Neural approaches learn representations directly~\citep{carriere2020perslay,xin2023gril,verma2024topological}; see~\citet{ali2023survey} for a survey. None of these methods provides hypothesis testing with interpretable effect sizes.

\paragraph{Statistical inference.}
\citet{robinson2017hypothesis} introduced a permutation test on Wasserstein distances, and \citet{islambekov2024vector} extended this to vectorised Betti functions. Confidence-set methods~\citep{fasy2014confidence} target single-diagram inference. Kernel-based two-sample tests apply maximum mean discrepancy~\citep{gretton2012kernel} via persistence-specific kernels~\citep{kwitt2015statistical,reininghaus2015stable,carriere2017sliced,schrab2023mmd}. Persistence kernels can in principle do both, but in practice the two are split, with a
kernel for classification, and a separate permutation or MMD test for inference. Across this landscape, every method yields a $p$-value but none yields an interpretable effect size or localises group differences to specific filtration scales (Appendix~\ref{app:landscape}, Table~\ref{tab:landscape}).

\paragraph{Survival analysis.}
Survival analysis models time-to-event data~\citep{clark2003survival}. The survival function $S(t) = \mathbb{P}(T > t)$ gives the probability that a random time $T$ exceeds $t$, and the Kaplan-Meier estimator~\citep{kaplan1958nonparametric} provides a non-parametric estimate from observed event times. The log-rank test~\citep{mantel1966evaluation} compares two survival distributions under $H^\text{null}: S_A = S_B$, with the median survival times and hazard ratios quantifying effect sizes. 

Further background on TDA and survival analysis appears in Appendix~\ref{app:bckgrd}.

\section{Survival Topological Representation ANalysis of Diagrams (\method)}\label{sec:mthds}

We introduce \method, that maps persistence diagrams (PDs) to survival data via the \emph{persistence survival function}. The same object underlies complementary capabilities: a non-parametric two-sample test for comparing topological structure between groups (Section~\ref{sec:hypothesis_testing}) along with effect sizes, and a fixed-dimensional vectorisation for downstream machine learning (Section~\ref{sec:vectorisation}). We first define the central object and establish its stability (Section~\ref{sec:p2s}); inference and vectorisation then inherit these guarantees.

\subsection{From Persistence to Survival}\label{sec:p2s}

Let $\mathcal{D} = \{(b_i, d_i)\}_{i=1}^n$ denote a PD, where each topological feature $i$ appears at filtration value $b_i$ and disappears at $d_i$. The \emph{persistence value} $p_i = d_i - b_i$ measures the lifetime of feature $i$ across the filtration. \method~treats each $p_i$ as a fully observed time-to-event and the multiset $\{p_i\}_{i=1}^n$ as a survival sample. The correspondence with classical survival analysis is detailed in Appendix~\ref{app:bridge}. We define the \textbf{persistence survival function}
\begin{equation}\label{eq:surv}
    S(t) = \mathbb{P}(p > t),
\end{equation}
the probability that a randomly chosen feature persists beyond filtration scale $t$. Equivalently, $S$ is determined by the cumulative hazard $\Lambda(t) = -\log S(t)$, which records the cumulative rate of feature disappearance up to $t$. $S$ takes values in $[0,1]$ and is the natural object for vectorisation (Section~\ref{sec:vectorisation}), while $\Lambda$ is additive across non-overlapping intervals, which is what makes the log-rank test of Section~\ref{sec:hypothesis_testing} a sum of independent contributions.

\paragraph{Non-parametric estimation.}
Given a PD with distinct ordered persistence values $t_1 < t_2 < \cdots < t_m$, let $n_j$ be the number of features at risk just before $t_j$ (i.e., features with $p_i \geq t_j$) and $e_j$ the number of events at $t_j$. The Kaplan--Meier estimator of the survival function and the Nelson--Aalen estimator of the cumulative hazard are
\begin{equation}\label{eq:km_standard}
    \hat S(t) = \prod_{t_j \leq t} \frac{n_j - e_j}{n_j},
    \qquad
    \hat\Lambda(t) = \sum_{t_j \leq t} \frac{e_j}{n_j},
\end{equation}
with variance estimated by Greenwood's formula $\widehat{\mathrm{Var}}(\hat S(t)) = \hat S(t)^2 \sum_{t_j \leq t} e_j / (n_j(n_j - e_j))$ \citep{kaplan1958nonparametric,aalen1978nonparametric}. PDs additionally exhibit a delayed-entry structure, with features only observable between birth and death. This is accommodated through left truncation, with the at-risk set redefined as $n_j = \#\{i : b_i < t_j \leq d_i\}$ . While Eq.~\eqref{eq:km_standard} corresponds to a marginal model on persistence values, the left-truncated form corresponds to a model on death times conditional on entry at birth. 

\paragraph{Estimands and sampling model.}
We distinguish two sampling levels. \emph{Within a single diagram}, the multiset $\{p_i\}_{i=1}^n$ defines an empirical persistence measure $\hat\mu_{\mathcal{D}} = \tfrac{1}{n}\sum_{i=1}^n \delta_{p_i}$, and $\hat S$ is its empirical survival function. \emph{Across diagrams sampled from a distribution $P$ on $\mathcal{D}$}, the pooled estimand is
\begin{equation}\label{eq:pooled_estimand}
    S_P(t) = \mathbb{E}_{D \sim P}\!\left[\, \mathbb{P}_{p \mid D}(p > t) \,\right],
\end{equation}
the marginal under the mixture of sampling a diagram from $P$ then a feature uniformly from that diagram, well-defined whenever $\mathbb{E}_P[|D|] < \infty$. The test of Section~\ref{sec:hypothesis_testing} compares $S_{P_A}$ to $S_{P_B}$.

\paragraph{Moments and functionals of persistence.}
As a useful consequence of the survival representation, the persistence survival function determines all integrable functionals of persistence through
\begin{equation}\label{eq:moments}
    \mathbb{E}[\,g(p)\,] = \int_0^\infty g'(t)\, S(t)\, dt
\end{equation}
for any absolutely continuous $g$ with $g(0) = 0$. Setting $g(t) = t^k$ gives $\mathbb{E}[p^k] = k\int_0^\infty t^{k-1} S(t)\, dt$; in particular $\mathbb{E}[p] = \int_0^\infty S(t)\, dt$, the mean persistence. Eq.~\eqref{eq:moments} shows that classical scalar summaries of a diagram (mean, variance, total persistence, $L^p$ norms of persistence) are linear functionals of $S$.

\paragraph{Stability.}
The survival vectorisation $\phi_S(\,\cdot\,;\mathbf{t}): \mathcal{D} \to [0,1]^k$, evaluating $\hat S$ on a fixed grid $\mathbf{t} = (t_1, \ldots, t_k)$, is Lipschitz stable in the 1-Wasserstein topology on PDs. Both the hypothesis test and the downstream classifier inherit this stability since both act on $\hat S$ through $\phi_S$. Full proofs are deferred to Appendix~\ref{app:proofs}.

\begin{theorem}[Stability]\label{thm:stability}
Let $D = \{(b_i, d_i)\}_{i=1}^n$ and $D' = \{(b'_j, d'_j)\}_{j=1}^m$ be PDs with $n, m \geq 1$ off-diagonal points. Then
\begin{equation}\label{eq:stability_main}
\|\phi_S(D;\mathbf{t}) - \phi_S(D';\mathbf{t})\|_\infty
\;\leq\;
\frac{4\, W_1(D, D')}{\delta_\Delta \cdot \max(n, m)}
\;+\;
\frac{2\, W_1(D, D')}{\delta \cdot \min(n, m)},
\end{equation}
where $W_1(D, D')$ is the 1-Wasserstein distance, $\delta$ is the minimum nonzero gap between matched persistence values, and $\delta_\Delta$ is the minimum persistence among features matched to the diagonal under any optimal $W_1$ transport. Both terms vanish as $W_1(D, D') \to 0$, so the bound is fully $W_1$-controlled regardless of cardinality. If there exists an optimal transport $\gamma^*$ such that $p_i = p'_{\gamma^*(i)}$ for all $i$, then $\phi_S(D;\mathbf{t}) = \phi_S(D';\mathbf{t})$.
\end{theorem}

\begin{corollary}[Continuity]\label{cor:continuity}
If the persistence values of $D$ are distinct ($p_i \neq p_j$ for $i \neq j$) and $W_1(D_n, D) \to 0$ for a sequence with $|D_n| = |D| = n$, then $\|\phi_S(D_n;\mathbf{t}) - \phi_S(D;\mathbf{t})\|_\infty \to 0$.
\end{corollary}



\subsection{Hypothesis Testing on (Collections of) Persistence Diagrams}\label{sec:hypothesis_testing}

The persistence survival function enables hypothesis testing on collections of persistence diagrams, with calibrated Type~I error, high power from a small number of diagrams, and effect sizes that go beyond a scalar p-value.

\paragraph{Two-sample test.}
Given two groups of PDs $\mathcal{D}_A = \{D_1^A, \ldots, D_{n_A}^A\}$ and $\mathcal{D}_B = \{D_1^B, \ldots, D_{n_B}^B\}$, \method~tests $H^\text{null} : S_{P_A}(t) = S_{P_B}(t)$ for all $t$ against $H^\text{alt.} : S_{P_A}(t) \neq S_{P_B}(t)$ for some $t$. We proceed as follows:
\begin{enumerate}[leftmargin=*]
    \item Pool persistence values within each group: $T_A = \{p_i : (b_i, d_i) \in D_j^A,\, j = 1, \ldots, n_A\}$, and analogously for $T_B$;
    \item Apply the log-rank test to compare $S_{P_A}(t)$ versus $S_{P_B}(t)$.
\end{enumerate}
At each distinct event time $t_j$, let $n_{Aj}$, $n_{Bj}$ denote the at-risk counts and $d_{Aj}$, $d_{Bj}$ the event counts in groups $A$ and $B$, with $n_j = n_{Aj} + n_{Bj}$ and $d_j = d_{Aj} + d_{Bj}$. Under $H^\text{null}$, the conditional expectation and hypergeometric variance of $d_{Aj}$ given the margins are
\begin{equation}\label{eq:logrank_moments}
    E_{Aj} = \frac{n_{Aj}\, d_j}{n_j}, \qquad
    V_j = \frac{n_{Aj}\, n_{Bj}\, d_j\, (n_j - d_j)}{n_j^2\, (n_j - 1)}.
\end{equation}
The log-rank statistic is
\begin{equation}\label{eq:logrank}
    \chi^2_{\mathrm{LR}} \;=\; \frac{\bigl(\sum_j (d_{Aj} - E_{Aj})\bigr)^2}{\sum_j V_j}
    \;\xrightarrow{\,H^\text{null}\,}\; \chi^2_1,
\end{equation}
a non-parametric statistic comparing observed to expected events under the null \citep{mantel1966evaluation}. \method~extends to $k > 2$ groups via the stratified log-rank statistic\footnote{The stratified log-rank statistic accounts for known sources of variation between strata (e.g., imaging site) by computing the test within each stratum and pooling.} with $k - 1$ degrees of freedom and standard pairwise corrections (e.g., Bonferroni). 

\paragraph{Effect sizes.}
Under proportional hazards, the hazard ratio $\mathrm{HR}$ such that $\lambda_B(t) = \mathrm{HR} \cdot \lambda_A(t)$ summarises the relative rate of feature disappearance between groups; $\widehat{\mathrm{HR}}$ is estimated non-parametrically from the at-risk and event counts in~\eqref{eq:logrank_moments}, with confidence intervals from a $1{,}000$-sample bootstrap over diagrams. We assess proportional hazards via Schoenfeld residuals~\citep{schoenfeld1982partial}. When proportional hazards assumption does not hold, we report instead the median persistence difference with $\Delta_{\mathrm{med}} = m_A - m_B, m_g = \inf\{t : S_g(t) \leq 1/2\}$, which remains a valid non-parametric effect-size summary without distributional assumptions.  Appendix~\ref{app:ph_failure} discusses procedures for the three regimes (PH holds, monotone deviation, non-monotone deviation).


\paragraph{Range normalization.}
When comparing PDs from different data instances, filtration ranges may vary. We adopt local normalization, i.e. each diagram's grid spans its own $[t_{\min}, t_{\max}]$. This preserves the relative shape of survival curves while allowing comparison across diagrams with different absolute scales. Global normalization (a fixed range across all diagrams) is an alternative when the full dataset is available at vectorisation time.

\paragraph{Statistical power.}
Topological complexity governs statistical power: manifolds with richer homology produce more features, increasing power at fixed effect size.

\begin{proposition}[Power under proportional hazards]\label{prop:power}
Assume (i) the persistence values pooled within each group form an i.i.d.\ sample from $S_{P_A}, S_{P_B}$ respectively, (ii) the proportional hazards model $\lambda_B = \mathrm{HR}\cdot \lambda_A$ holds, and (iii) equal allocation $n_A = n_B$. Then the asymptotic power of the log-rank test at level $\alpha$ against the alternative $\mathrm{HR}$ is
\[
    1 - \Phi\!\left(z_{1-\alpha/2} - \frac{|\log \mathrm{HR}|\,\sqrt{n_{\mathrm{eff}}}}{2}\right) + o(1),
\]
where $n_{\mathrm{eff}} := |PD^A| + |PD^B|$ is the total number of persistence features across both groups \citep{schoenfeld1982partial}.
\end{proposition}

For point clouds of $N$ points sampled from a $d$-dimensional manifold $\mathcal{M}$ at density $\rho$, $\mathbb{E}[|PD_k|] = O(\rho^{d-k}\, \beta_k(\mathcal{M}))$, where $\beta_k$ is the $k$-th Betti number \citep{giusti2015clique}. For graphs with $n$ vertices and edge density $p$, degree-based filtrations yield $\mathbb{E}[|PD_0|] = n - 1$ and $\mathbb{E}[|PD_1|] = O(p n^2 - n)$, providing concrete guidance for power calculations. The same scaling that yields high power from a small number of diagrams also raises the question of whether large feature pools produce statistically significant but practically negligible 
differences; we address this in Appendix~\ref{app:hypersensitivity}.

\subsection{\method~vectorisation}\label{sec:vectorisation}

Because $S$ is a one-dimensional summary, it admits a natural fixed-dimensional discretization. \method's vectorisation uses the same $\hat S$ that drives the hypothesis test, so the features fed to a classifier are exactly the quantities tested for significance.

\paragraph{Grid discretization.}
Given a PD $\mathcal{D}$ with estimated survival function $\hat{S}_\mathcal{D}$, we evaluate at $g$ equally-spaced points spanning the observed persistence range:
\begin{equation}
    t_k = t_{\min} + \frac{k-1}{g-1}(t_{\max} - t_{\min}), \quad k = 1, \ldots, g.
\end{equation}
The \method~vector for homology dimension $h$ is $\mathbf{v}^{(h)} = \bigl[\hat{S}_\mathcal{D}^{(h)}(t_1),\, \ldots,\, \hat{S}_\mathcal{D}^{(h)}(t_g)\bigr] \in \mathbb{R}^g$. For PDs with multiple homology dimensions we concatenate: $\mathbf{v} = \bigl[\mathbf{v}^{(0)} \,\|\, \mathbf{v}^{(1)} \,\|\, \cdots\bigr] \in \mathbb{R}^{g \cdot |\mathcal{H}|}$. The grid resolution $g$ controls the trade-off between representational fidelity and feature dimensionality.

\paragraph{Variants and their estimands.}
The variants estimate distinct functionals of the diagram-generating process:
\begin{itemize}
    \item \textbf{\method~(standard, primary).} Uses the Kaplan-Meier estimator of Eq. \eqref{eq:km_standard} on persistence values, targeting the marginal $S(t) = \mathbb{P}(p > t)$. Each feature is treated as entering at time zero. This variant produces $g \cdot |\mathcal{H}|$ features and is the default in all experiments unless otherwise stated.

    \item \textbf{\method-F (birth features, augmented primary).} Augments the standard vector with a birth survival curve, $\hat{S}_{\mathrm{birth}}(t) = \mathbb{P}(b > t)$, capturing the joint $(b,d)$ distribution through its two marginals. The final vector is $\mathbf{v}_{\method\text{-}F} = [\mathbf{v}_{\mathrm{persistence}} \,\|\, \mathbf{v}_{\mathrm{birth}}] \in \mathbb{R}^{2g \cdot |\mathcal{H}|}$. \method-F is well-defined only when birth times vary; for $H_0$ from sublevel-set or Vietoris-Rips filtrations on connected complexes, births are identically zero, and \method-F reduces to \method.

    \item \textbf{\method-B (birth entry, alternative).} Uses the left-truncated Kaplan-Meier estimator, with features entering the risk set at their birth time. The estimand is the conditional death-time survival $S_{B}(t) = \mathbb{P}(d > t \mid b < t)$, related to a hazard $\lambda_d(t) = \lim_{\Delta\to 0} \mathbb{P}(d \in [t, t+\Delta) \mid d \geq t,\, b < t)/\Delta$ on the death-time axis. The marginal formulation is appropriate when persistence itself is the structural object of interest, which is typical in TDA classification (see Appendix~\ref{app:strand_b} for empirical analysis of why the conditional estimand underperforms in classification).
\end{itemize}

\paragraph{Complexity.}
For a PD with $n$ points, the Kaplan-Meier estimator requires $O(n \log n)$ time (dominated by sorting persistence values); grid evaluation is $O(g)$ via binary search. For $|\mathcal{H}|$ homology dimensions, total per-diagram complexity is $O(|\mathcal{H}| \cdot (n \log n + g))$. \method's practical advantage stems from low output dimensionality: with $g = 25$ and two homology dimensions, \method~produces 50 features (or 100 for \method-F). This reduces downstream classifier training cost, particularly for kernel methods whose complexity scales with feature dimension.

\section{Experiments}\label{sec:expe}
We evaluate \method~on three fronts: (i) statistical validation on synthetic manifolds with controlled topology, verifying Type~I error and power; (ii) classification benchmarks on graph and 3D point cloud datasets, showing the survival vectorisation is competitive with state-of-the-art representations; and (iii) an application to autism brain connectivity, where \method~quantifies developmental topological differences.

\subsection{Hypothesis Testing Validation}\label{sec:hyp_test_results}

\begin{table}[t]
\begin{minipage}{0.52\textwidth}
    \centering
    \caption{Type~I error control at nominal level $\alpha = 0.05$ ($n=200$ independent replications per setting). $\hat{\alpha}$ is the observed false positive rate.}
    \label{tab:type1_summary}
    \small
    \begin{tabular}{llc}
        \toprule
        Setting & Dim & $\hat{\alpha}$ [95\% CI] \\
        \midrule
        Parametric distributions & -- & 0.056 [0.044, 0.069] \\
        \midrule
        Torus & $H_0$ & 0.025 [0.011, 0.057] \\
         & $H_1$ & 0.075 [0.046, 0.120] \\
        Klein Bottle  & $H_0$ & 0.000 [0.000, 0.019] \\
         & $H_1$ & 0.015 [0.005, 0.043] \\
        Sphere & $H_0$ & 0.025 [0.011, 0.057] \\
         & $H_1$ & 0.025 [0.011, 0.057] \\
        \bottomrule
    \end{tabular}
\end{minipage}
\hfill
\begin{minipage}{0.42\textwidth}
    \centering
    \caption{Statistical power of \method~on simulated 4D torus data at $\alpha = 0.05$.}
    \label{tab:sample_efficiency}
    \small
    \begin{tabular}{llcc}
        \toprule
        Experiment & Size & $H_0$ & $H_1$ \\
        \midrule
        \multirow{4}{*}{Point cloud}
        & $N=30$ & 0.80 & 0.32 \\
        & $N=50$ & 1.00 & 0.67 \\
        & $N=100$ & 1.00 & 0.78 \\
        & $N=150$ & 1.00 & 0.81 \\
        \midrule
        \multirow{4}{*}{Collection}
        & $n=2$ & 1.00 & 0.91 \\
        & $n=3$ & 1.00 & 0.98 \\
        & $n=5$ & 1.00 & 1.00 \\
        & $n=10$ & 1.00 & 1.00 \\
        \bottomrule
    \end{tabular}
\end{minipage}
\end{table}

We validate \method's inference capabilities through controlled experiments on synthetic manifolds, demonstrating proper Type~I error control, statistical power, and the ability to detect differences from few samples.

\paragraph{Setup.} 
Following~\citet{stolz2023outlier}, we generate point clouds containing mixtures of signal points (sampled from known manifolds) and noise points. We evaluate on three manifolds: a 4D torus, a 4D Klein bottle, and a 3D sphere with various noise configurations. For each setting, we vary the signal probability $p_{\text{signal}} \in [0,1]$ across 50 values with 10 independent realizations each, yielding 500 datasets per manifold. Full simulation details are provided in Appendix \ref{app:type1}.

\paragraph{Type~I Error Control.}
To verify that \method~maintains nominal false positive rates, we test under $H^\text{null}$ where both groups are drawn from identical distributions. Table~\ref{tab:type1_summary} summarizes results across parametric distributions and topological manifolds. In all cases, the observed Type~I error rate is consistent with or conservative relative to the nominal $\alpha = 0.05$. Topological features arising from the same diagram may be positively correlated through the elder rule, potentially inflating the variance estimate of the asymptotic log-rank test. We validate empirically that this dependence does not distort inference at the nominal level. For each replicate across the six manifold scenarios above, we compare the asymptotic log-rank $p$-value against an assumption-free permutation $p$-value (1{,}000 permutations of group labels at the diagram level). The two tests produced concordant reject/retain decisions in $\geq 99\%$ of cases across all 16 unique conditions tested, with Spearman $\rho \geq 0.99$ between the two $p$-values (full per-condition results in Appendix~\ref{app:permutation_calibration}). Where they differ, observed rejection rates are at or below $\alpha$, confirming that within-diagram dependence is conservative for the asymptotic test in practice.

\paragraph{Sample Efficiency.} 
A practical advantage of topological methods is their ability to reveal meaningful structure from limited data, as PH captures global shape properties even from small point clouds, enabling statistically powerful comparisons with few samples. \method~achieves high power from few samples by leveraging the abundance of topological features within each diagram. Table~\ref{tab:sample_efficiency} demonstrates this on simulated 4D torus data across two regimes: comparing two point clouds of size $N$ (one diagram each), and comparing two groups of $n$ diagrams each. Power exceeds 90\% on $H_1$ from as few as $n=2$ diagrams per group.

\paragraph{Power to Detect Topological Differences.}
\begin{wrapfigure}{r}{0.56\textwidth}
    \vspace{-\baselineskip}
    \centering
    \includegraphics[width=0.56\textwidth]{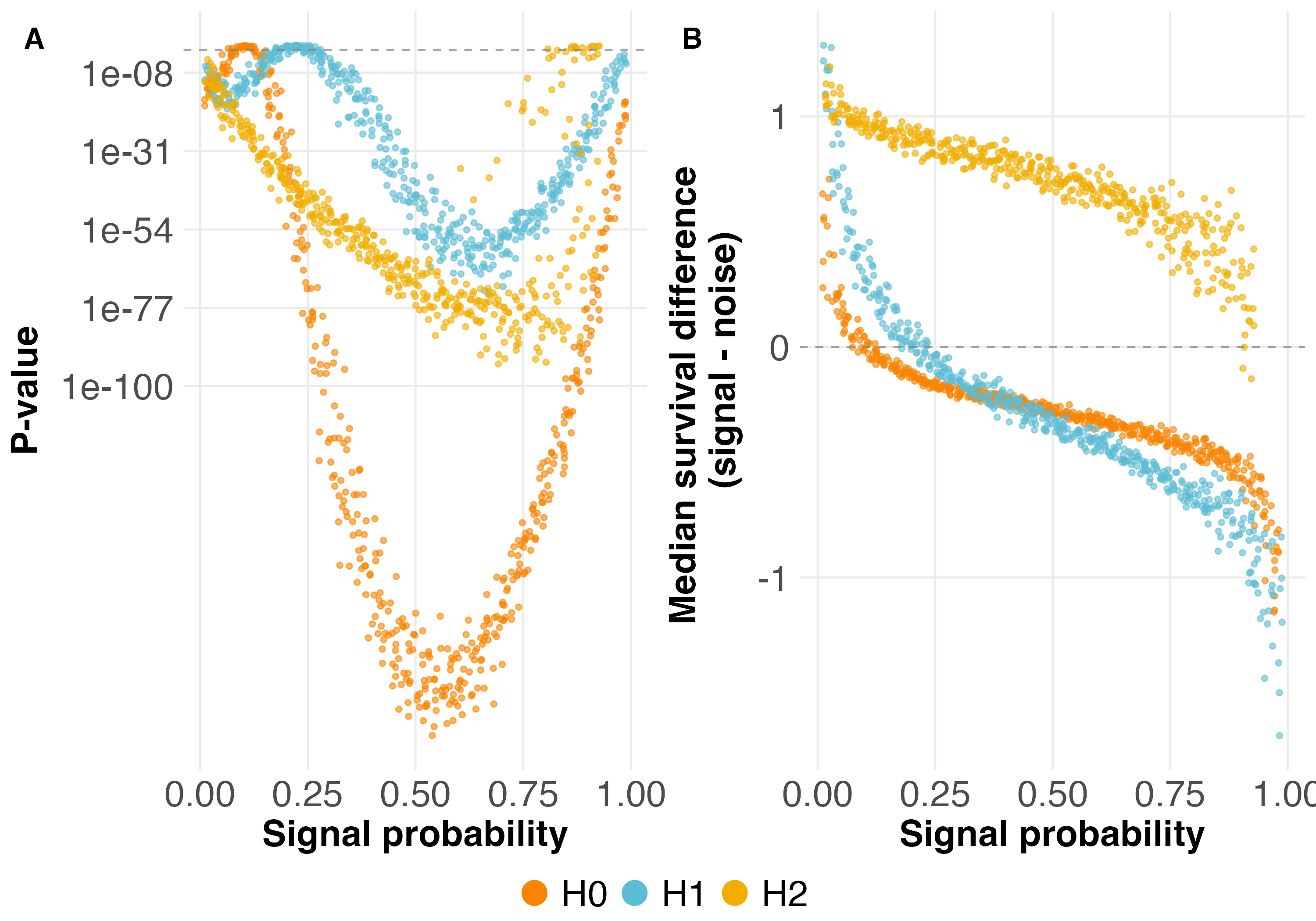}
    \caption{Hypothesis testing validation on 4D torus (500 simulations). A: Log-rank $p$-values comparing signal vs.\ noise survival distributions; dashed line indicates $\alpha = 0.05$. B: Median survival difference quantifies effect size.}
    \label{fig:hyp_validation}
    \vspace{-0.8cm}
\end{wrapfigure}
We assess \method's ability to distinguish signal (manifold) from noise features using log-rank tests on pooled persistence values. Figure \ref{fig:hyp_validation} illustrates results on the 4D torus across the signal probability spectrum. Across all homology dimensions, \method~achieves strong discrimination, with $p$-values consistently below 0.05 throughout the signal probability range. The median survival difference shows interpretable patterns: at low signal probability, genuine topological features persist substantially longer than noise; this gap narrows as signal dominates the point cloud. Power as a function of hazard ratio on parametric distributions follows the expected V-shape, with power $\approx 0$ at $\text{HR} = 1$ and $> 95\%$ at $\text{HR} \geq 1.5$ (Appendix \ref{app:power}, Figure \ref{fig:power}).
 

\subsection{Classification Benchmark}\label{sec:vectorisation_results}
 
While hypothesis testing is \method's primary contribution, the same survival function provides a fixed-dimensional vector representation suitable for downstream machine learning. We benchmark this vectorisation across graph and 3D point cloud classification at scale.
 
\paragraph{Setup.}
We evaluate on ten graph datasets from the TU Dortmund collection~\citep{morris2020tudataset}, and on four 3D point cloud datasets from the flooder package~\citep{graf2025flood} (datasets description in Appendix \ref{app:bench_results}, Table \ref{tab:datasets}). Following~\citet{carriere2020perslay}, we compute PDs using degree-based filtration~\citep{horn2021topological} in dimensions $H_0$ and $H_1$ for each graph. For point clouds, we use the flood complex~\citep{graf2025flood} with $H_0$, $H_1$, and $H_2$. We use Support Vector Machine  and Random Forest classifiers across all vectorisations, with stratified 10-fold cross-validation (graphs) or the 10 pre-defined splits provided by~\citet{graf2025flood} (point clouds). 

We compare \method~against a comprehensive set of vectorisation and kernel baselines: Persistence Images (PI)~\citep{adamsPersistenceImagesStable2017}, Persistence Landscapes (PL)~\citep{bubenik2017persistence}, Betti Curves (BC)~\citep{giusti2015clique}, Silhouette and Entropy summaries from the GUDHI library, PersLay (fixed, used as a vectorizer for fair comparison)~\citep{carriere2020perslay}, and two persistence kernels: Sliced Wasserstein Kernel (SWK)~\citep{carriere2017sliced} and Persistence Scale Space Kernel (PSSK)~\citep{reininghaus2015stable}. Hyperparameters were selected by grid search on the training portion of each split (Appendix~\ref{app:hyperparams}).

\begin{table*}[t]
\centering
\caption{Classification accuracy (\%) with SVM. \method~reports the best variant per dataset (full per-variant results in Appendix~\ref{app:bench_results}). \textbf{Bold} indicates best result per row; \underline{underline} indicates second best.}
\label{tab:main_results}
\small
\begin{tabular*}{\textwidth}{l@{\extracolsep{\fill}}ccccccccc}
\toprule
Dataset & \method & PI & PL & BC & Silh. & Entr. & PLay & SWK & PSSK \\
\midrule
\multicolumn{10}{@{}l}{\textit{TU Dortmund graphs (degree filtration, $H_0/H_1$; 10-fold CV)}} \\
\midrule
COX2 & 78.5 & 78.7 & 78.2 & \underline{78.8} & 78.2 & 78.2 & \textbf{78.9} & 78.2 & 78.2 \\
DD & 71.3 & 75.0 & 72.1 & \textbf{75.4} & 69.4 & 69.4 & 74.8 & \textbf{75.4} & 74.3 \\
DHFR & \underline{67.6} & 64.3 & 61.0 & 62.6 & 60.9 & 63.4 & 61.2 & \textbf{68.8} & 61.0 \\
ENZYMES & \textbf{27.0} & 24.3 & 25.2 & 22.0 & \underline{25.7} & 21.0 & 23.4 & 21.1 & 21.2 \\
IMDB-B & \underline{64.2} & 62.5 & 63.7 & 63.4 & \textbf{64.4} & \underline{64.2} & 63.5 & 63.7 & 63.6 \\
IMDB-M & \textbf{41.3} & \underline{40.3} & 39.2 & 38.4 & 39.8 & 37.1 & 40.0 & 40.1 & 39.2 \\
MUTAG & 77.9 & \textbf{78.8} & 73.8 & 78.4 & 68.3 & 65.8 & 75.2 & \textbf{78.8} & 66.5 \\
NCI1 & 62.6 & 63.0 & 60.4 & 62.0 & 57.1 & 57.1 & \underline{63.1} & \textbf{63.2} & 61.2 \\
PROTEINS & 66.4 & \underline{71.0} & 70.6 & 70.0 & 66.4 & 67.3 & 70.8 & \textbf{71.3} & 60.6 \\
PTC-MR & \textbf{58.5} & 55.3 & 54.7 & 53.6 & 56.8 & 56.3 & 56.2 & 55.7 & \underline{57.6} \\
\midrule
Average & \underline{61.5} & 61.3 & 59.9 & 60.5 & 58.7 & 58.0 & 60.7 & \textbf{61.6} & 58.3 \\
\midrule
\multicolumn{10}{@{}l}{\textit{3D point cloud datasets (flood complex, $H_0/H_1/H_2$; 10 pre-defined splits)}} \\
\midrule
SwissCheese & \textbf{100.0} & 97.9 & 88.1 & 98.5 & 98.7 & \underline{98.8} & -- & -- & -- \\
Corals & \textbf{70.8} & 70.0 & 66.0 & \underline{70.2} & 61.9 & 64.4 & 63.1 & -- & -- \\
MCB & \textbf{81.9} & 78.9 & \underline{80.3} & 69.0 & 46.6 & 51.4 & -- & -- & -- \\
ModelNet10 & \textbf{65.5} & \underline{59.9} & 55.3 & 52.5 & 45.3 & 55.7 & -- & -- & -- \\
\midrule
Average & \textbf{79.5} & \underline{76.7} & 72.4 & 72.5 & 63.1 & 67.6 & 63.1 & -- & -- \\
\bottomrule
\end{tabular*}
\\[0.5em]
{\footnotesize
PI = Persistence Images, PL = Persistence Landscapes, BC = Betti Curves, PLay = PersLay (fixed), SWK = Sliced Wasserstein Kernel, PSSK = Persistence Scale Space Kernel. Kernel methods (SWK, PSSK) are computationally infeasible on the larger flooder datasets due to $O(N^2)$ Gram matrices and per-pair complexity scaling with diagram size.}
\vspace{-0.5cm}
\end{table*}

\paragraph{Results.}
Table~\ref{tab:main_results} shows that on TU Dortmund graphs, all methods cluster within a narrow accuracy range (58--63\% on average), reflecting the limited topological signal of degree filtration on these small graphs~\citep{morris2020tudataset}. \method, SWK, and PI lead with 61.5\%, 61.6\%, 61.3\%; pairwise differences are not statistically significant (Wilcoxon signed-rank, all $p > 0.5$). On the geometrically richer point cloud benchmarks, \method~achieves the highest average accuracy (79.5\% vs.\ 76.7\% for PI, $+2.8$pp), with the largest margins on the most complex multi-class datasets (MCB $+3.0$pp, ModelNet10 $+5.6$pp). To our knowledge, this is the first systematic application of TDA vectorisations to point clouds at this scale. Across the 14 datasets, \method~achieves the best accuracy on 7 and is within 1--5pp of the best baseline on the remaining 7 (see Appendix~\ref{app:strand_b} for \method~variants discussion). 

\paragraph{Computational Efficiency.}
\begin{wraptable}{r}{0.42\textwidth}
    \vspace{-\baselineskip}
    \centering
    \caption{Computational efficiency on the TU Dortmund average. Features: output dimensionality. Time: SVM pipeline (single AMD EPYC 9654 CPU node).}
    \label{tab:efficiency}
    \small
    \begin{tabular}{lrr}
    \toprule
    Method & Features & Time (s) \\
    \midrule
    \method~($g=25$) & 58 & 6.1 \\
    \method-F ($g=25$) & 112 & 18.9 \\
    \midrule
    PI & 5{,}000 & 188.8 \\
    PL & 500 & 15.4 \\
    BC & 100 & 6.9 \\
    \bottomrule
    \end{tabular}
\end{wraptable}
\method's advantage beyond accuracy is feature parsimony and computational tractability. Table~\ref{tab:efficiency} compares feature dimensionality and SVM runtime on the TU Dortmund average. \method~uses 58--168 features versus 5{,}000 for PI on a $50\times50$ grid, a 30--85$\times$ reduction with no loss in accuracy. The advantage compounds on flooder datasets where kernel methods become infeasible: SWK and PSSK require $O(N^2)$ Gram matrices with per-pair cost scaling in diagram size. On ModelNet10 ($N=4{,}899$, flood complexes on 250{,}000-point clouds), SWK is estimated at ${\sim}10^4$ hours versus 258 seconds for \method. PersLay's TensorFlow fitting similarly scales prohibitively. \method's cost is $O(n \log n)$ per diagram with no pairwise comparisons.

\subsection{Case study on functional connectivity data in autism}\label{sec:abide}

Prior topological work on developmental connectomes has established that connectome topology carries developmentally meaningful signal~\citep{gracia2023development}. We illustrate \method~on characterising topological differences between children's and adults' brain networks. 

\paragraph{Setup.}
We use the Autism Brain Imaging Data Exchange (ABIDE~I)~\citep{heinsfeld2018identification}, restricted to typically developing controls, and contrast children (age $<$ 13, $n = 130$) with adults (age $\geq$ 18, $n = 110$). 
For each subject and atlas, we build the Pearson correlation matrix between ROI resting-state time series\footnote{Pearson correlation is one of several functional-connectivity measures. The choice can influence network structure~\citep{bullmore2011brain,smith2011network,zhou2009matlab}. We use Pearson correlation throughout for compatibility with the developmental-connectivity literature we anchor against~\citep{gracia2023development}.} 
To assess sensitivity to ROI partition, we replicate under three brain parcellations of differing granularity, AAL90 (90 ROIs)~\citep{tzourio2002automated}, CC200 (200 ROIs)~\citep{craddock2012whole}, and Schaefer-400 (400 ROIs)~\citep{schaefer2018local}. From each distance matrix we compute persistence diagrams in $H_0$ and $H_1$ via Vietoris--Rips filtration, yielding one PD per subject per dimension. We then test whether the children's and adults' collections of PDs differ, and characterise the difference if one is detected.

\paragraph{Inference and comparison methods.}
For each parcellation and homological dimension, \method~returns a log-rank $p$-value (stratified by imaging site) along with a hazard ratio (HR) with 95\% CI. We compare against three established PD-inference baselines on the same diagrams, each with $B = 1{,}000$ label permutations: the Wasserstein-permutation test of \citet{robinson2017hypothesis} (R\&T), and kernel maximum mean discrepancy (MMD) with the Sliced Wasserstein Kernel~\citep{carriere2017sliced} (MMD-SWK) and the Persistence Scale Space Kernel~\citep{reininghaus2015stable} (MMD-PSSK). 

\paragraph{Findings.}
Table~\ref{tab:abide_primary} reports results across atlases and methods. \method~detects a developmental difference at $H_1$ on every parcellation, with HR $> 1$ in every case, 
indicating that loops in children's diagrams have shorter persistence than in adults'. The HR magnitudes give a quantitative read of the effect, with the increasing magnitude at finer parcellations consistent with finer ROI partitions resolving more of the distinguishing topological structure (further results are reported in Appendix~\ref{app:abide_robustness}). At $H_0$, the result is parcellation-dependent. \method~detects the difference at Schaefer-400 (HR $= 0.815$, CI $[0.731, 0.909]$, $p = 6.3 \times 10^{-3}$) but not at coarser parcellations, and two of the three comparator methods (R\&T, MMD-PSSK) replicate this pattern, indicating the $H_0$ signal lives at the parcellation grid resolution rather than at a methodological choice. 

\begin{table}[t]
\centering
\caption{ABIDE developmental contrast (children $<$ 13 vs. adults $\geq$ 18) $p$-values across three atlases. \method~is reported with HR (95\% CI). \textbf{Bold} indicates significance at $\alpha = 0.05$.}
\label{tab:abide_primary}
\small
\setlength{\tabcolsep}{4pt}
\begin{tabular}{l l l l l l}
\toprule
Dim & Atlas & \method & R\&T & MMD-SWK & MMD-PSSK \\
\midrule
\multirow{3}{*}{$H_0$}
& Schaefer-400 & $\mathbf{0.006}$ (0.815 [0.731, 0.909]) & $\mathbf{0.007}$ & 0.182          & $\mathbf{0.011}$ \\
& CC200        & 0.801 (0.982 [0.882, 1.094])           & 0.492             & 0.523          & 0.482            \\
& AAL90        & 0.622 (0.967 [0.878, 1.066])           & 0.268             & $\mathbf{0.025}$ & 0.602          \\
\midrule
\multirow{3}{*}{$H_1$}
& Schaefer-400 & $\mathbf{<0.001}$ (1.145 [1.106, 1.185]) & $\mathbf{0.008}$ & $\mathbf{0.001}$ & $\mathbf{0.001}$ \\
& CC200        & $\mathbf{0.001}$ (1.101 [1.060, 1.143])  & $\mathbf{0.008}$ & $\mathbf{0.008}$ & $\mathbf{0.001}$ \\
& AAL90        & $\mathbf{0.021}$ (1.090 [1.041, 1.142])  & $\mathbf{0.019}$ & $\mathbf{0.010}$ & $\mathbf{0.001}$ \\
\bottomrule
\end{tabular}
\vspace{-0.5cm}
\end{table}

\section{Conclusion}

\method~is designed for settings where a binary p-value is not enough --- where practitioners need to know how much two cohorts differ. Such settings are widespread, e.g., drug discovery, where compound libraries differ topologically in ring systems ($H_1$) and binding pockets ($H_2$) by quantifiable amounts, biomedical imaging, where confirmatory studies require sample-size calculations grounded in effect-size estimates, and the analysis of neural network activation manifolds~\citep{gardinazzi2024persistent,bhan2025did}, where topological characterisations of training and fine-tuning regimes are an emerging direction. In each of these, \method~replaces the current dichotomy between vectorisation pipelines (predictive but silent on inference) and pairwise tests (inferential but silent on magnitude and direction) with a single object that does both.

\bibliography{bibli}
\bibliographystyle{plainnat}

\medskip






\appendix



\section{Bridging TDA and survival analysis}\label{app:bridge}
The central innovation of our approach lies in recognizing that the birth-death lifecycle of topological features constitutes a time-to-event process amenable to survival analysis. To illustrate this mapping, consider a simple example:

\noindent \paragraph{Example 3.1 (Circle with noise).} Suppose we have a point cloud sampled from a circle with added Gaussian noise, as per Figure \ref{fig:bridge}. As we increase the filtration parameter $\epsilon$ from 0, we observe the following feature lifecycle:
\begin{itemize}
    \item $\epsilon = 0.1$: Many small connected components appear (birth events);
    \item $\epsilon = 0.3$: Components begin merging; the main circular structure emerges as a 1-dimensional loop (loop birth);
    \item $\epsilon = 0.5$: Most noise components die (death events for components); the main loop persists;
    \item $\epsilon = 0.8$: The circular loop finally dies when the "hole" gets filled (loop death).
\end{itemize}

In survival analysis terminology, each connected component and loop is a "subject" entering the study at its birth time and experiencing a "death event" when it disappears. Figure \ref{fig:bridge} illustrates this conceptual mapping. Topological features become subjects with birth times as entry points and death times as events, while filtration parameters provide temporal structure. This mapping enables statistical comparison between subgroups (signal and noise in TDA, or treatment groups in clinical studies) using established survival methodologies.

This reframing fundamentally changes how we approach the signal-noise discrimination problem. This enables both discovery of signal-noise structure and classification of individual features using well-established statistical methodology. 

\begin{figure}
    \centering
    \includegraphics[width=1\linewidth]{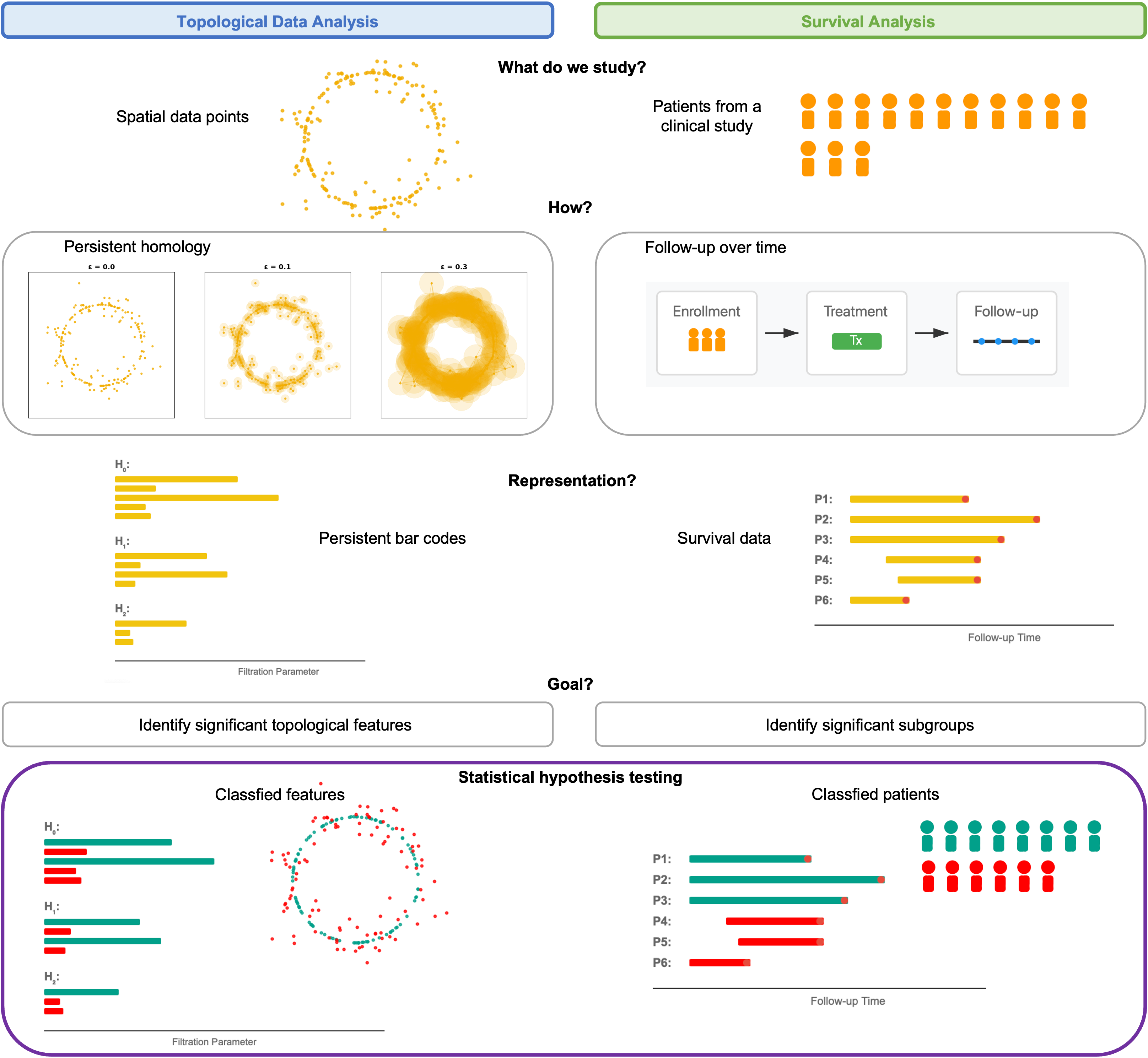}
    \caption{Bridging topological data and survival analysis fields with \texttt{\method}. Our approach operates on spatial data containing genuine topological structure and random noise artifacts. Through persistent homology, we track how topological features emerge and disappear across increasing filtration parameters. Simultaneously, in survival analysis, patients progress through clinical enrollment, treatment, and follow-up phases, generating survival timelines. In both fields, the goal is usually to identify significant populations, either features in TDA, or responders in clinical studies.}
    \label{fig:bridge}
\end{figure}

\section{Further background and related work}\label{app:bckgrd}

\subsection{Persistent homology}

\paragraph{Simplicial complexes.}
A \emph{simplicial complex} $K$ is a collection of simplices (vertices, edges, triangles, tetrahedra, and their higher-dimensional analogues) closed under taking faces: if $\sigma \in K$ and $\tau \subseteq \sigma$, then $\tau \in K$. The $k$-skeleton of $K$ consists of all simplices of dimension at most $k$. Formally, a $k$-simplex is the convex hull of $k+1$ affinely independent points; for example, a 0-simplex is a vertex, a 1-simplex is an edge, and a 2-simplex is a triangle.

A \emph{filtration} is a nested sequence of simplicial complexes $\emptyset = K_0 \subseteq K_1 \subseteq \cdots \subseteq K_n = K$, where simplices are added incrementally. The filtration value $f(\sigma)$ of a simplex $\sigma$ is the smallest index $i$ such that $\sigma \in K_i$. Common filtrations include:
\begin{itemize}[noitemsep,topsep=0pt,leftmargin=*]
    \item \emph{Vietoris-Rips filtration}: Given a point cloud $X \subset \mathbb{R}^d$ and scale parameter $\epsilon$, the Rips complex $\text{VR}_\epsilon(X)$ includes a simplex $\{x_0, \ldots, x_k\}$ if all pairwise distances satisfy $\|x_i - x_j\| \leq \epsilon$.
    \item \emph{Čech filtration}: The Čech complex $\check{C}_\epsilon(X)$ includes a simplex if the corresponding $\epsilon$-balls have non-empty intersection. This is homotopy equivalent to the union of balls but more expensive to compute.
    \item \emph{Alpha filtration}: A subset of the Delaunay triangulation, providing a sparser representation with the same persistent homology as the Čech complex.
\end{itemize}

\paragraph{Persistence diagrams from point cloud data.}
Given a point cloud $X = \{x_1, \ldots, x_n\} \subset \mathbb{R}^d$, persistent homology tracks how topological features evolve across scales. The standard pipeline proceeds as follows:
\begin{enumerate}[noitemsep,topsep=0pt,leftmargin=*]
    \item \emph{Filtration construction}: Build a filtered simplicial complex, typically the Vietoris-Rips filtration $\{\text{VR}_\epsilon(X)\}_{\epsilon \geq 0}$.
    \item \emph{Boundary matrix computation}: Represent the boundary operator $\partial_k: C_k \to C_{k-1}$ as a matrix, where $C_k$ is the vector space of $k$-chains.
    \item \emph{Matrix reduction}: Apply the standard reduction algorithm~\citep{edelsbrunner2002topological} to compute the persistence pairs, identifying which simplices create (birth) and destroy (death) homology classes.
    \item \emph{Diagram extraction}: Output the PD $\mathcal{D} = \{(b_i, d_i)\}$, where each point represents a feature born at filtration value $b_i$ and dying at $d_i$.
\end{enumerate}
The resulting diagram is stable: if two point clouds $X$ and $X'$ satisfy $d_H(X, X') \leq \epsilon$ (Hausdorff distance), then $W_\infty(\mathcal{D}_X, \mathcal{D}_{X'}) \leq \epsilon$~\citep{cohen2005stability}.

\paragraph{Persistence diagrams from graphs.}
For graph data $G = (V, E)$, several filtration strategies produce meaningful PDs:
\begin{itemize}[noitemsep,topsep=0pt,leftmargin=*]
    \item \emph{Degree filtration}~\citep{hoferDeepLearningTopological2017}: Vertices enter the filtration in order of increasing degree, and edges appear when both endpoints are present. This captures hierarchical connectivity: hub vertices (high degree) appear late, while peripheral vertices appear early.
    \item \emph{Heat kernel signature (HKS) filtration}~\citep{carriere2020perslay}: Vertices are ordered by the diagonal of the heat kernel $H_t = e^{-tL}$, where $L$ is the graph Laplacian. This requires eigendecomposition and depends on the diffusion time $t$.
    \item \emph{Sublevel set filtration on node attributes}: When vertices have scalar attributes $f: V \to \mathbb{R}$, the sublevel set filtration adds vertices in order of $f(v)$ and edges when both endpoints are present.
\end{itemize}

For graph filtrations, the $H_0$ features correspond to connected components merging as the filtration progresses, while $H_1$ features correspond to cycles forming and being filled in. The degree filtration used in this work is parameter-free and has been shown to capture chemically meaningful patterns in molecular graphs~\citep{horn2021topological}.

\subsection{Topological Deep Learning}
Recent work integrates topological information into neural architectures. Higher-order message passing on simplicial and cell complexes~\citep{bodnar2021weisfeiler,pmlr-v139-bodnar21a,hajij2025copresheaf} generalizes GNNs beyond pairwise interactions. Persistence-augmented GNNs~\citep{horn2021topological,immonen2023going} incorporate topological features as auxiliary inputs, with applications to molecular property prediction~\citep{xin2025toping,wang2025topotein}. Deep learning with topological signatures~\citep{hoferDeepLearningTopological2017} and topological autoencoders~\citep{moor2020topological} demonstrate the value of persistence-based representations; see~\citet{papamarkou2024position} for a position paper on the field. \method~complements these architectures: its features can augment any GNN as graph-level descriptors while additionally providing statistical inference capabilities.

\subsection{Survival Analysis in Machine Learning}
Survival analysis models time-to-event data using the Kaplan-Meier estimator~\citep{kaplan1958nonparametric} for nonparametric survival estimation and the Cox proportional hazards model~\citep{cox1972regression} for covariate analysis. The Cox model assumes the hazard function takes the form $h(t|x) = h_0(t) \exp(\beta^\top x)$, where $h_0(t)$ is a baseline hazard and $\beta$ are regression coefficients. The proportional hazards assumption states that the hazard ratio between any two individuals is constant over time.

Modern extensions include random survival forests~\citep{ishwaran2008random,murris2025random}, which aggregate survival trees using random feature selection and bootstrap aggregation, and neural approaches such as DeepSurv~\citep{katzman2018deepsurv}, which replaces the linear predictor with a neural network, and DeepHit~\citep{lee2018deephit}, which directly models the probability mass function of event times. 

The log-rank test~\citep{mantel1966evaluation} is the standard nonparametric test for comparing survival distributions between groups. Under the null hypothesis of equal survival functions, the test statistic
\begin{equation}
    \chi^2 = \frac{\left(\sum_j (O_{1j} - E_{1j})\right)^2}{\sum_j V_j}
\end{equation}
follows a $\chi^2_1$ distribution, where $O_{1j}$ and $E_{1j}$ are observed and expected events in group 1 at time $t_j$, and $V_j$ is the variance.

\subsection{Positioning our work}\label{app:landscape}

Table~\ref{tab:landscape} situates \method~within the existing landscape of 
methods for analysing persistence diagrams. Vectorisation methods (top block) 
produce fixed-dimensional representations suitable for downstream machine 
learning but provide no native hypothesis test. Statistical-inference 
methods (middle block) return $p$-values via permutation or kernel procedures 
but provide no effect-size summary and no localisation of where, along 
the persistence axis, two cohorts differ. \method~is the only entry that 
ticks all four columns: a stable vectorisation, a hypothesis test, an 
interpretable effect size, and per-scale localisation through the 
Kaplan-Meier curve. The closest comparator is the Betti-permutation test 
of \citet{islambekov2024vector}, which combines a vectorisation with 
a permutation test on the same object; the differentiator for \method~is 
the effect-size and per-scale-localisation machinery that the survival 
representation makes available natively.

\begin{table*}[t]
    \caption{Comparison of methods for analyzing persistence diagrams. \method~is the only method that simultaneously supports vectorisation for downstream ML, hypothesis testing, and interpretable effect sizes. Methods above the line are vectorisations; methods below the line are statistical tests.}
    \label{tab:landscape}
    \centering
    \footnotesize                        
    \setlength{\tabcolsep}{4pt}          
    \begin{tabular}{@{}lllcccc@{}}       
    \toprule
    Method & Representation & Stability & Vec. & Test & Eff.\ Size & Per-scale \\
    \midrule
    Betti Curves~\citep{giusti2015clique}
        & Feature counts & $\infty$-Wass.\ & \cmark & \xmark & \xmark & \xmark \\
    Pers.\ Images~\citep{adamsPersistenceImagesStable2017}
        & Gauss.-smoothed grid & 1-Wass.\ & \cmark & \xmark & \xmark & \xmark \\
    Pers.\ Landscapes~\citep{bubenik2017persistence}
        & Piecewise-linear fn. & $\infty$-Wass.\ & \cmark & Ltd.$^\dagger$ & \xmark & \xmark \\
    Pers.\ Kernels~\citep{reininghaus2015stable}
        & RKHS embedding & 1-Wass.\ & \cmark & \xmark$^\S$ & \xmark & \xmark \\
    PersLay~\citep{carriere2020perslay}
        & Learned aggregation & Cond.$^\ddagger$ & \cmark & \xmark & \xmark & \xmark \\
    GRIL~\citep{xin2023gril}
        & Gen.\ rank invariant & 1-Lip.\ & \cmark & \xmark & \xmark & \xmark \\
    \midrule
    Wass.\ permutation~\citep{robinson2017hypothesis}
        & Pairwise distances & --- & \xmark & \cmark & \xmark & \xmark \\
    Betti permutation~\citep{islambekov2024vector}
        & Vec.\ Betti function & 1-Wass.\ & \cmark & \cmark & \xmark & \xmark \\
    \midrule
    \textbf{\method~(Ours)}
        & Pers.\ survival fn. & 1-Wass.\ & \cmark & \cmark & \cmark & \cmark \\
    \bottomrule
    \multicolumn{7}{@{}l@{}}{\footnotesize $^\dagger$Provides $z$-tests and permutation tests on landscape norms, but without effect size interpretation.} \\
    \multicolumn{7}{@{}l@{}}{\footnotesize $^\S$Kernels admit a two-sample test only through an external MMD wrapper, with no native effect size.} \\
    \multicolumn{7}{@{}l@{}}{\footnotesize $^\ddagger$1-Wasserstein stability with specific weight functions; 2-Lipschitz w.r.t.\ filtration parameters.}
    \end{tabular}
    \vskip -0.1in
\end{table*}

\section{Proofs}\label{app:proofs}

\subsection{Proof of Theorem~\ref{thm:stability} (Stability)}\label{app:stability}

We provide a complete proof of the stability theorem for \method~vectorisation.

\subsubsection{Preliminaries and Notation}

\begin{definition}[Persistence Diagram]
A persistence diagram is a multiset $D = \{(b_i, d_i)\}_{i=1}^n \subset \{(b,d) \in \mathbb{R}^2 : b \leq d\}$, together with the diagonal $\Delta = \{(x,x) : x \in \mathbb{R}\}$ with infinite multiplicity.
\end{definition}

\begin{definition}[Wasserstein Distance]
The 1-Wasserstein distance between persistence diagrams $D$ and $D'$ is:
\begin{equation}
W_1(D, D') = \inf_\gamma \sum_{x \in D \cup \Delta} \|x - \gamma(x)\|_\infty,
\end{equation}
where the infimum is over all bijections $\gamma : D \cup \Delta \to D' \cup \Delta$.
\end{definition}

\begin{definition}[Persistence Values]
For a persistence diagram $D = \{(b_i, d_i)\}_{i=1}^n$, the persistence values are $p_i = d_i - b_i$ for $i = 1, \ldots, n$.
\end{definition}

\begin{definition}[Empirical Survival Function]
The empirical survival function on persistence values is:
\begin{equation}
\hat{S}_D(t) = \frac{1}{n} \sum_{i=1}^n \mathbf{1}\{p_i > t\} = \frac{|\{i : p_i > t\}|}{n}.
\end{equation}
\end{definition}

\begin{definition}[\method~vectorisation]
For grid points $0 \leq t_1 < t_2 < \cdots < t_k$, the \method~survival vectorisation is:
\begin{equation}
\phi_S(D) = \left(\hat{S}_D(t_1), \hat{S}_D(t_2), \ldots, \hat{S}_D(t_k)\right) \in [0,1]^k.
\end{equation}
\end{definition}

\subsubsection{Supporting Lemmas}

\begin{lemma}[2-Lipschitz Projection]\label{lem:lipschitz}
The persistence projection $\pi : (b,d) \mapsto d - b$ is 2-Lipschitz with respect to the $\ell^\infty$ norm:
\begin{equation}
|p_1 - p_2| = |(d_1 - b_1) - (d_2 - b_2)| \leq 2 \cdot \|(b_1, d_1) - (b_2, d_2)\|_\infty
\end{equation}
for any $(b_1, d_1), (b_2, d_2) \in \mathbb{R}^2$. 
\end{lemma}

\begin{proof}
We have:
\begin{align*}
|(d_1 - b_1) - (d_2 - b_2)| &= |(d_1 - d_2) - (b_1 - b_2)| \\
&\leq |d_1 - d_2| + |b_1 - b_2| \quad \text{(triangle inequality)}\\
&\leq 2 \max(|d_1 - d_2|, |b_1 - b_2|) \\
&= 2 \|(b_1, d_1) - (b_2, d_2)\|_\infty.
\end{align*}

Consider $(b_1, d_1) = (0, 1)$ and $(b_2, d_2) = (\epsilon, 1 - \epsilon)$ for small $\epsilon > 0$. Then:
\begin{itemize}
    \item $\|(b_1, d_1) - (b_2, d_2)\|_\infty = \max(\epsilon, \epsilon) = \epsilon$,
    \item $|p_1 - p_2| = |1 - (1 - 2\epsilon)| = 2\epsilon$.
\end{itemize}
\end{proof}

\begin{lemma}[Diagonal Matching Cost]\label{lem:diagonal}
For a point $(b,d)$ with $d > b$ matched to the diagonal, the optimal matching cost is:
\begin{equation}
\inf_{(c,c) \in \Delta} \|(b,d) - (c,c)\|_\infty = \frac{d-b}{2} = \frac{p}{2},
\end{equation}
achieved at $c = \frac{b+d}{2}$. Equivalently, the persistence value satisfies $p = 2 \cdot (\text{matching cost})$.
\textit{Note:} This is a standard result in optimal transport on persistence diagrams; we include the proof for completeness.
\end{lemma}

\begin{proof}
For any $c \in \mathbb{R}$:
\begin{equation}
\|(b,d) - (c,c)\|_\infty = \max(|b-c|, |d-c|).
\end{equation}
This is minimized when $|b-c| = |d-c|$, which occurs at $c = \frac{b+d}{2}$. The minimum value is:
\begin{equation}
\left|b - \frac{b+d}{2}\right| = \frac{d-b}{2} = \frac{p}{2}. \qedhere
\end{equation}
\end{proof}

\begin{lemma}[Transport Cost on Matched Persistence Values Values]\label{lem:transport}
Let $D = \{(b_i, d_i)\}_{i=1}^n$ and $D' = \{(b'_j, d'_j)\}_{j=1}^m$ be persistence diagrams, and let $\gamma^*$ be an optimal matching for $W_1(D, D')$. Define matched persistence values for each $i \in \{1, \ldots, n\}$:
\begin{equation}
p^{(\mathrm{ext})}_i = 
\begin{cases}
p'_j = d'_j - b'_j & \text{if } (b_i, d_i) \leftrightarrow (b'_j, d'_j) \text{ under } \gamma^* \\
0 & \text{if } (b_i, d_i) \leftrightarrow \Delta \text{ under } \gamma^*
\end{cases}
\end{equation}
Then:
\begin{equation}
\sum_{i=1}^{n} |p_i - p^{(\mathrm{ext})}_i| \leq 2 \cdot W_1(D, D').
\end{equation}
\end{lemma}

\begin{proof}
We analyze the contribution from each type of match in $\gamma^*$.

\textbf{Case 1: Off-diagonal to off-diagonal.} 
If $(b_i, d_i) \leftrightarrow (b'_j, d'_j)$, then by Lemma~\ref{lem:lipschitz}:
\begin{equation}
|p_i - p'_j| \leq 2 \|(b_i, d_i) - (b'_j, d'_j)\|_\infty = 2 \cdot (\text{matching cost for this pair}).
\end{equation}

\textbf{Case 2: Off-diagonal to diagonal.}
If $(b_i, d_i) \leftrightarrow \Delta$, then $p^{(\mathrm{ext})}_i = 0$ and the contribution is:
\begin{equation}
|p_i - 0| = p_i = 2 \cdot \frac{p_i}{2} = 2 \cdot (\text{matching cost for this pair}),
\end{equation}
by Lemma~\ref{lem:diagonal}.

\textbf{Summing over all matches:}
\begin{equation}
\sum_{i=1}^{n} |p_i - p^{(\mathrm{ext})}_i| \leq 2 \cdot \sum_{\text{matches involving } D} (\text{matching cost}) \leq 2 \cdot W_1(D, D'). \qedhere
\end{equation}
\end{proof}

\subsubsection{Main Proof}

\begin{proof}[Proof of Theorem~\ref{thm:stability}]

\textbf{Step 1: Setup and Matching Partition.}

Let $\gamma^* : D \cup \Delta \to D' \cup \Delta$ be an optimal matching achieving $W_1(D, D')$. Partition the off-diagonal points according to their matching:
\begin{itemize}
    \item $\mathcal{A} = \{(i,j) : (b_i, d_i) \leftrightarrow (b'_j, d'_j)\}$ --- off-diagonal to off-diagonal,
    \item $\mathcal{B} = \{i : (b_i, d_i) \leftrightarrow \Delta\}$ --- points of $D$ matched to diagonal,
    \item $\mathcal{C} = \{j : \Delta \leftrightarrow (b'_j, d'_j)\}$ --- points of $D'$ matched from diagonal.
\end{itemize}

Let $a = |\mathcal{A}|$, so $|\mathcal{B}| = n - a$ and $|\mathcal{C}| = m - a$.

\textit{Note:} In practice, when the Wasserstein distance is small relative to persistence values, most off-diagonal points are matched off-diagonal to off-diagonal, with diagonal matching reserved for cardinality differences or very short-lived features.

\textbf{Step 2: Matched Persistence Values.}

For each $i \in \{1, \ldots, n\}$, define matched persistence values as in Lemma~\ref{lem:transport}:
\begin{equation}
p^{(\mathrm{ext})}_i = 
\begin{cases}
p'_j & \text{if } (b_i, d_i) \leftrightarrow (b'_j, d'_j) \in D' \\
0 & \text{if } (b_i, d_i) \leftrightarrow \Delta
\end{cases}
\end{equation}

\textit{Note:} Even when $n = m$, we may have $|\mathcal{B}| = |\mathcal{C}| > 0$. That is, equal cardinality does \emph{not} imply all points are matched off-diagonal to off-diagonal.

\textbf{Step 3: Decompose the Survival Function Difference.}

For any threshold $t \geq 0$, let $N_D(t) = |\{i : p_i > t\}|$ and $N_{D'}(t) = |\{j : p'_j > t\}|$. Then:
\begin{equation}
\hat{S}_D(t) = \frac{N_D(t)}{n}, \qquad \hat{S}_{D'}(t) = \frac{N_{D'}(t)}{m}.
\end{equation}

We decompose the difference as:
\begin{align}
\hat{S}_D(t) - \hat{S}_{D'}(t) &= \frac{N_D(t)}{n} - \frac{N_{D'}(t)}{m} \\
&= \underbrace{\left(\frac{N_D(t)}{n} - \frac{N_D(t)}{m}\right)}_{\text{Term (I): normalization difference}} + \underbrace{\left(\frac{N_D(t) - N_{D'}(t)}{m}\right)}_{\text{Term (II): count difference}}.
\end{align}

\textbf{Step 4a: Bound Term (I) --- Normalization Difference.}

\begin{equation}
\left| \frac{N_D(t)}{n} - \frac{N_D(t)}{m} \right| = N_D(t) \left| \frac{1}{n} - \frac{1}{m} \right| = N_D(t) \cdot \frac{|m - n|}{nm}.
\end{equation}

Since $N_D(t) \leq n$:
\begin{equation}
|\text{Term (I)}| \leq n \cdot \frac{|m - n|}{nm} = \frac{|m - n|}{m} = \left| 1 - \frac{n}{m} \right|.
\end{equation}

\textbf{Step 4b: Bound the cardinality difference via diagonal matching.}

The bound $|\text{Term (I)}| \leq |1 - n/m|$ from Step 4 can be tightened by observing that the cardinality difference itself is controlled by $W_1$. Without loss of generality assume $m \geq n$ (the case $n > m$ is symmetric). In any optimal matching $\gamma^*$ achieving $W_1(D, D')$, the set $\mathcal{C} = \{j : \Delta \leftrightarrow (b'_j, d'_j)\}$ defined in Step~1 satisfies $|\mathcal{C}| = m - a \geq m - n$, since at most $a \leq n$ points of $D'$ can be matched off-diagonal. By Lemma~\ref{lem:diagonal}, each diagonally-matched point $(b'_j, d'_j)$ incurs cost $p'_j / 2$ in $W_1$. Therefore:
\begin{equation}
W_1(D, D') \;\geq\; \sum_{j \in \mathcal{C}} \frac{p'_j}{2} \;\geq\; |\mathcal{C}| \cdot \frac{\delta_\Delta}{2} \;\geq\; (m - n) \cdot \frac{\delta_\Delta}{2},
\end{equation}
where $\delta_\Delta = \min_{j \in \mathcal{C}} p'_j$. Rearranging:
\begin{equation}
\left| 1 - \frac{n}{m} \right| \;=\; \frac{m - n}{m} \;\leq\; \frac{2 \, W_1(D, D')}{m \cdot \delta_\Delta} \;=\; \frac{2 \, W_1(D, D')}{\max(n,m) \cdot \delta_\Delta}.
\end{equation}

\textbf{Step 5: Bound Term (II) --- Count Difference via Straddling.}

We relate $N_D(t)$ and $N_{D'}(t)$ through the matching $\gamma^*$.

\textbf{Step 5a: Express counts using the matching structure.}

\begin{align}
N_D(t) &= \sum_{i=1}^{n} \mathbf{1}\{p_i > t\} \\
&= \sum_{(i,j) \in \mathcal{A}} \mathbf{1}\{p_i > t\} + \sum_{i \in \mathcal{B}} \mathbf{1}\{p_i > t\}.
\end{align}

\begin{align}
N_{D'}(t) &= \sum_{j=1}^{m} \mathbf{1}\{p'_j > t\} \\
&= \sum_{(i,j) \in \mathcal{A}} \mathbf{1}\{p'_j > t\} + \sum_{j \in \mathcal{C}} \mathbf{1}\{p'_j > t\}.
\end{align}

Therefore:
\begin{align}
N_D(t) - N_{D'}(t) &= \sum_{(i,j) \in \mathcal{A}} \left(\mathbf{1}\{p_i > t\} - \mathbf{1}\{p'_j > t\}\right) \label{eq:A-contribution}\\
&\quad + \sum_{i \in \mathcal{B}} \mathbf{1}\{p_i > t\} - \sum_{j \in \mathcal{C}} \mathbf{1}\{p'_j > t\}. \label{eq:BC-contribution}
\end{align}

\textbf{Step 5b: Unify using matched persistence values.}

Recall that for $i \in \mathcal{B}$, we have $p^{(\mathrm{ext})}_i = 0$, so:
\begin{equation}
\mathbf{1}\{p_i > t\} - \mathbf{1}\{p^{(\mathrm{ext})}_i > t\} = \mathbf{1}\{p_i > t\} - \mathbf{1}\{0 > t\} = \mathbf{1}\{p_i > t\}
\end{equation}
since $\mathbf{1}\{0 > t\} = 0$ for all $t \geq 0$.

Thus:
\begin{equation}
\sum_{i=1}^{n} \left(\mathbf{1}\{p_i > t\} - \mathbf{1}\{p^{(\mathrm{ext})}_i > t\}\right) = \sum_{(i,j) \in \mathcal{A}} \left(\mathbf{1}\{p_i > t\} - \mathbf{1}\{p'_j > t\}\right) + \sum_{i \in \mathcal{B}} \mathbf{1}\{p_i > t\}.
\end{equation}

\textbf{Step 5c: Define the straddling set.}

\begin{definition}[Straddling]
A pair $(p_i, p^{(\mathrm{ext})}_i)$ \emph{straddles} threshold $t$ if exactly one value exceeds $t$:
\begin{equation}
\mathcal{S}(t) = \left\{ i \in \{1, \ldots, n\} : \min(p_i, p^{(\mathrm{ext})}_i) \leq t < \max(p_i, p^{(\mathrm{ext})}_i) \right\}.
\end{equation}
\end{definition}

We also claim $\left| \mathbf{1}\{p_i > t\} - \mathbf{1}\{p^{(\mathrm{ext})}_i > t\} \right| = 1$ if and only if $i \in \mathcal{S}(t)$.

\begin{proof}[Proof of Claim]
The indicator difference equals 1 exactly when one value exceeds $t$ and the other does not:
\begin{center}
\begin{tabular}{cccc}
\toprule
$p_i$ vs $t$ & $p^{(\mathrm{ext})}_i$ vs $t$ & $|\mathbf{1}\{p_i > t\} - \mathbf{1}\{p^{(\mathrm{ext})}_i > t\}|$ & Straddles? \\
\midrule
$p_i > t$ & $p^{(\mathrm{ext})}_i > t$ & 0 & No \\
$p_i > t$ & $p^{(\mathrm{ext})}_i \leq t$ & 1 & Yes \\
$p_i \leq t$ & $p^{(\mathrm{ext})}_i > t$ & 1 & Yes \\
$p_i \leq t$ & $p^{(\mathrm{ext})}_i \leq t$ & 0 & No \\
\bottomrule
\end{tabular}
\end{center}
The straddling condition $\min(p_i, p^{(\mathrm{ext})}_i) \leq t < \max(p_i, p^{(\mathrm{ext})}_i)$ holds exactly in rows 2 and 3.
\end{proof}

\textbf{Step 5d: Bound using straddling count.}

From Step 5b:
\begin{equation}
\left| \sum_{i=1}^{n} \left(\mathbf{1}\{p_i > t\} - \mathbf{1}\{p^{(\mathrm{ext})}_i > t\}\right) \right| \leq \sum_{i=1}^{n} \left| \mathbf{1}\{p_i > t\} - \mathbf{1}\{p^{(\mathrm{ext})}_i > t\} \right| = |\mathcal{S}(t)|.
\end{equation}

\textbf{Step 5e: Account for points in $\mathcal{C}$.}

The count $N_{D'}(t)$ includes contributions from points in $\mathcal{C}$ (matched from diagonal), which are not captured by the sum over $i \in \{1, \ldots, n\}$. We have:
\begin{equation}
N_D(t) - N_{D'}(t) = \sum_{i=1}^{n} \left(\mathbf{1}\{p_i > t\} - \mathbf{1}\{p^{(\mathrm{ext})}_i > t\}\right) - \sum_{j \in \mathcal{C}} \mathbf{1}\{p'_j > t\}.
\end{equation}

Therefore:
\begin{equation}
|N_D(t) - N_{D'}(t)| \leq |\mathcal{S}(t)| + |\mathcal{C}|.
\end{equation}

\textbf{Step 5f: Bound $|\mathcal{S}(t)|$ and $|\mathcal{C}|$.}

For $i \in \mathcal{S}(t)$, we have $p_i \neq p^{(\mathrm{ext})}_i$. Let $\delta = \min_{i: p_i \neq p^{(\mathrm{ext})}_i} |p_i - p^{(\mathrm{ext})}_i|$. By Lemma~\ref{lem:transport}, each index with $p_i \neq p^{(\mathrm{ext})}_i$ contributes at least $\delta$ to $\sum_i |p_i - p^{(\mathrm{ext})}_i| \leq 2 W_1(D, D')$, so:
\begin{equation}
|\mathcal{S}(t)| \leq \frac{2 W_1(D, D')}{\delta}.
\end{equation}

For $|\mathcal{C}|$: each $j \in \mathcal{C}$ corresponds to $(b'_j, d'_j) \in D'$ matched from the diagonal, contributing $p'_j / 2$ to $W_1$ by Lemma~\ref{lem:diagonal}. With $\delta_\Delta = \min_{j \in \mathcal{C}} p'_j$:
\begin{equation}
|\mathcal{C}| \cdot \frac{\delta_\Delta}{2} \leq W_1(D, D'), 
\quad\text{so}\quad
|\mathcal{C}| \leq \frac{2 W_1(D, D')}{\delta_\Delta}.
\end{equation}

\textbf{Step 6: Combine the Bounds.}

By symmetry of the bound in $D$ and $D'$, assume without loss of generality $m = \max(n, m)$. From Step 4b:
\begin{equation}
|\text{Term (I)}| \leq \frac{2 W_1(D, D')}{\delta_\Delta \cdot \max(n, m)}.
\end{equation}
For Term (II), Steps 5e--5f give $|N_D(t) - N_{D'}(t)| \leq |\mathcal{S}(t)| + |\mathcal{C}|$. Dividing by $m = \max(n,m)$ and applying Step 5f:
\begin{equation}
|\text{Term (II)}| \leq \frac{2 W_1(D, D')}{\delta \cdot \max(n,m)} + \frac{2 W_1(D, D')}{\delta_\Delta \cdot \max(n, m)}.
\end{equation}
Combining and using $\max(n,m) \geq \min(n,m)$ to weaken the $\delta$ term:
\begin{align}
|\hat S_D(t) - \hat S_{D'}(t)| 
&\leq |\text{Term (I)}| + |\text{Term (II)}| \\
&\leq \frac{4 W_1(D, D')}{\delta_\Delta \cdot \max(n, m)} + \frac{2 W_1(D, D')}{\delta \cdot \min(n, m)}.
\end{align}

\textbf{Step 7: Conclude.}

Taking the maximum over the grid points $t_1, \ldots, t_k$:
\begin{equation}
\|\phi_S(D;\mathbf{t}) - \phi_S(D';\mathbf{t})\|_\infty \leq \frac{4 W_1(D, D')}{\delta_\Delta \cdot \max(n, m)} + \frac{2 W_1(D, D')}{\delta \cdot \min(n, m)}.
\end{equation}
Both terms vanish as $W_1(D, D') \to 0$, establishing continuity in the 1-Wasserstein topology for all diagram pairs.

\end{proof}

\subsubsection{Proof of Corollary~\ref{cor:continuity}}

\begin{proof}
Let $\{D_k\}_{k=1}^\infty$ be a sequence of PDs with $|D_k| = |D| = N$ and $W_1(D_k, D) \to 0$.

By Theorem~\ref{thm:stability}, with $\delta_k$ and $\delta_\Delta^k$ defined for the optimal matching between $D_k$ and $D$:
\begin{equation}
\|\phi_S(D_k) - \phi_S(D)\|_\infty \leq \frac{4 W_1(D_k, D)}{\delta_\Delta^k \cdot N} + \frac{2 W_1(D_k, D)}{\delta_k \cdot N}.
\end{equation}

As $W_1(D_k, D) \to 0$, the optimal matching $\gamma_k^*$ eventually pairs each point in $D_k$ with a unique nearby point in $D$ (diagonal matching becomes suboptimal when points are close), so $\mathcal{C} = \emptyset$ for $k$ sufficiently large and the first term vanishes identically.

Since the persistence values of $D$ are distinct, there exists $\delta_{\min} > 0$ such that $|p_i - p_j| \geq \delta_{\min}$ for all $i \neq j$ in $D$. For sufficiently large $k$, $\delta_k \geq \delta_{\min} / 2$, giving:
\begin{equation}
\|\phi_S(D_k) - \phi_S(D)\|_\infty \leq \frac{4 W_1(D_k, D)}{\delta_{\min} \cdot N} \to 0.
\end{equation}
\end{proof}

\subsubsection{Remarks on the Stability Bound}

\begin{remark}[Role of the constants]
The bound contains two constants: $4$ on the $\delta_\Delta$ component and $2$ on the $\delta$ component. The factor of 2 in both arises from two structural facts: (i) the persistence projection $(b,d) \mapsto d - b$ is 2-Lipschitz with respect to the $\ell^\infty$ norm (Lemma~\ref{lem:lipschitz}), and (ii) for diagonal matching, the persistence value equals twice the Wasserstein cost: $p = 2 \cdot (p/2)$ (Lemma~\ref{lem:diagonal}). The factor of 4 on the $\delta_\Delta$ component reflects two independent contributions: $\mathcal{C}$ controls both the cardinality difference (Term (I), Step 4b) and part of the count difference (Term (II), Step 5f), each contributing $2 W_1 / (\delta_\Delta \cdot \max(n,m))$.
\end{remark}

\begin{remark}[Data-dependent quantities $\delta$ and $\delta_\Delta$]
The bound depends on two data-dependent quantities: $\delta$, the minimum nonzero gap between matched persistence values, and $\delta_\Delta$, the minimum persistence among diagonally-matched features. Both can in principle be arbitrarily small, but in practice:
\begin{itemize}
    \item For point clouds from continuous distributions, persistence values are generically distinct with $\delta > 0$.
    \item The number of pairs straddling any threshold is bounded by $2W_1/\delta$, providing meaningful control even when $\delta$ is small.
    \item $\delta_\Delta$ is bounded below by the smallest persistence in the diagram, which for Vietoris-Rips filtrations is controlled by the sampling density.
\end{itemize}
\end{remark}

\begin{remark}[Comparison with other methods]
The stability of \method~is comparable in form to other vectorisation methods:
\begin{itemize}
    \item Persistence Images: $\|PI(D) - PI(D')\|_\infty \leq C(\sigma) \cdot W_1(D, D')$, where $C(\sigma) \to \infty$ as the kernel bandwidth $\sigma \to 0$~\citep{adamsPersistenceImagesStable2017}.
    \item Persistence Landscapes: $\|\lambda(D) - \lambda(D')\|_\infty \leq W_\infty(D, D')$~\citep{bubenik2017persistence}; 1-Lipschitz with respect to $W_\infty$, which is stronger than $W_1$ but vacuous for diagrams of differing cardinality.
\end{itemize}
\method's bound is fully $W_1$-controlled and remains finite for unequal-cardinality diagrams without requiring a kernel bandwidth, with data-dependent dependence on $\delta$ and $\delta_\Delta$ analogous to how PI's constant depends on $\sigma$.
\end{remark}

\section{Type I Error Control: Full Results}
\label{app:type1}

We first evaluate our approach using survival analysis to distinguish topological signal from noise based on death times across homology groups $H_0$ (connected components), $H_1$ (loops), and $H_2$ (voids).  Figure \ref{fig:torus1} shows results on an illustrative example of a 4D torus with fixed signal probability. We display Kaplan-Meier survival curves on all topological features (A) showing the proportion of features remaining alive at each filtration parameter $\epsilon$. Figure \ref{fig:torus1} B displays Kaplan-Meier survival curves for each signal-noise features, with associated p-value from the log-rank test. Figure \ref{fig:torus1} C presents a complementary perspective via death time density distributions. In $H_0$ and $H_1$, signal features significantly outlive noise features ($p < 0.0001$, log-rank test). For voids in $H_2$, we see that signal voids die at the most advanced filtration. 

\paragraph{Strong discrimination across signal regimes.} Figure \ref{fig:torus1}A shows pronounced discrimination between signal and noise populations across all three homological dimensions. All $H_0$, $H_1$, and  $H_2$ exhibit consistently strong discrimination with log-rank $p$-values below the significance threshold ($p < 0.05$) throughout the signal probability cases, in particular with balanced signal-noise ratios ($0.40-0.60$ signal probability).

\paragraph{Effect size patterns reveal survival differences.} Figure \ref{fig:torus1}B quantifies the magnitude of survival differences through median survival differences (signal minus noise). All three dimensions exhibit decreasing patterns as signal probability increases. In $H_0$ and $H_1$, the transition from positive to negative differences signal probability (around $0.10$ and $0.20$, respectively) indicates that in low signal cases, noise features die earlier than signal features, while in high signal cases, the reverse occurs. When signal dominates, signal features may merge/die earlier due to increased connectivity. $H_2$ shows predominantly positive differences, indicating that signal features consistently die at a higher filtration compared to noise features. Notably, median differences approaching zero align with non-significant $p$-values, confirming the relationship between effect size and statistical significance.

Whether topological features from genuine geometric structures die at systematically different filtration scales compared to noise features, these illustrative results provide affirmative evidence across all homological dimensions ($H_0$, $H_1$, and $H_2$). Signal and noise features exhibit fundamentally different survival characteristics. 

\begin{figure}
    \centering
    \includegraphics[width=1\linewidth]{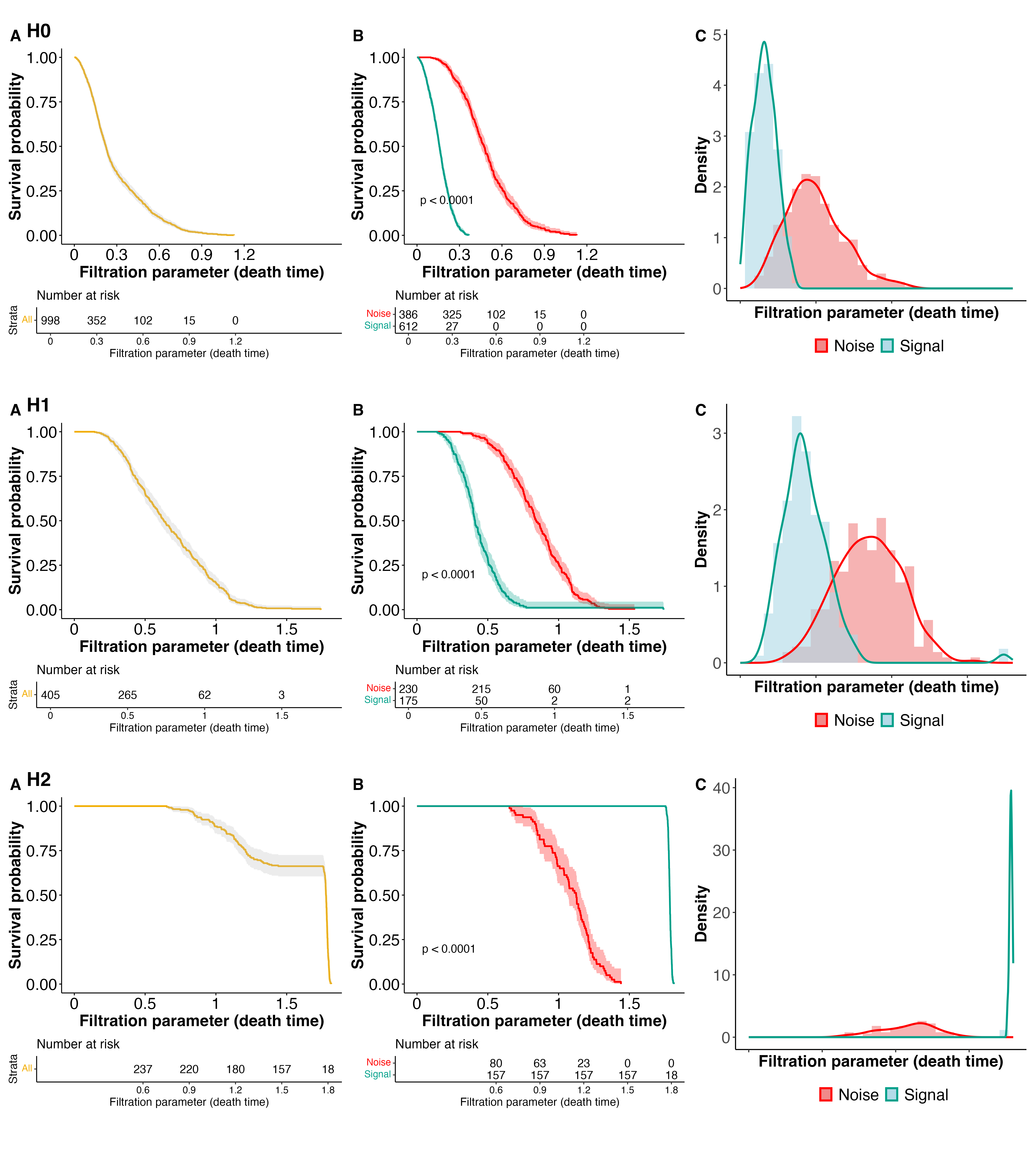}
    \caption{Kaplan-Meier estimators of topological features in 4D torus data with signal parameter of 0.6. Each row corresponds to a different homology dimension computed by persistent homology (e.g., dimension 0 for connected components, 1 for loops, and 2 for voids). (A) Overall survival curves with confidence intervals and risk set tables. (B) Stratified survival curves comparing signal vs. noise components, with $p$-values from log-rank tests. (C) Density distributions of filtration parameters (death times) for both signal and noise components.}
    \label{fig:torus1}
\end{figure}

Following the simulation settings from \cite{stolz2023outlier}, we evaluate \texttt{\method} on synthetic point clouds with controlled signal-noise ratios. 

\paragraph{Signal.} We generate synthetic point clouds containing mixtures of signal points (lying on known topological structures) and noise points. For each dataset, we control the signal probability $p_{signal} \in [0,1]$, which determines the proportion of points that belong to the underlying topological structure versus noise.
\textbf{3D Sphere}: Signal points are uniformly distributed on the unit sphere $S^2 \subset \mathbb{R}^3$, generated by normalizing random vectors from a 3D Gaussian distribution. This provides a canonical 2-dimensional manifold with known homology (single connected component, no 1-loop, single void).
\textbf{4D Torus}: Signal points lie on a 2-torus embedded in $\mathbb{R}^4$ using the parametrization: $(\cos(\gamma), \sin(\gamma), \cos(\phi), \sin(\phi))$, where $\gamma, \phi \in (0, 2\pi)$ are uniformly distributed angles. This flat torus embedding creates a more complex test case with characteristic topological signature (single connected component, two independent 1-loops, single void).
\textbf{4D Klein Bottle}: Signal points are generated using the Klein bottle parametrization:
\begin{align*}
    x_1 &= \cos(\gamma)(r\cos(\phi) + C) \\
    x_2 &= \sin(\gamma)(r\cos(\phi) + C) \\
    x_3 &= \cos(\gamma/2) \cdot r\sin(\phi) \\
    x_4 &= \sin(\gamma/2) \cdot r\sin(\phi)
\end{align*}
with $r = 3$, $C = 2$, and $\gamma, \phi \in (0, 2\pi)$. The Klein bottle represents a non-orientable surface with known homology (single connected component, single 1-loop, no voids), testing method robustness on non-orientable manifolds.

\paragraph{Noise.} To assess robustness across different noise characteristics, we implement four distinct noise models for the sphere experiments, each probing different aspects of the signal-noise discrimination problem:
Cube noise points are uniformly distributed in the cube $[-1, 1]^3$, creating ambient noise that completely surrounds the sphere signal.
Plane noise points are confined to the $xy$-plane, uniformly distributed in $[-3, 3]^2 \times {0}$, creating structured noise in a lower-dimensional subspace that intersects the sphere.
Line noise points lie on the $x$-axis with coordinates $(\alpha, 0, 0)$ where $\alpha$ is uniformly distributed in $[-50, 50]$, representing highly structured 1D noise geometrically distant from the 2D sphere signal.
Laplace noise points follow the same line structure but with $\alpha$ drawn from a Laplace distribution $\text{Laplace}(\mu=4, \sigma=0.5)$, clipped to $[-50, 50]$, testing robustness to different probability distributions within the same geometric structure.
For the torus experiments, noise points follow the same parametric form as signal points but with radius $r$ uniformly distributed in $(0, 2)$ instead of being fixed at 1, creating manifold-respecting noise that occupies a different region of the embedding space.
For Klein bottle experiments, noise points use the same parametrization but with $r \in [2, 4]$ and $C \in [1, 3]$ uniformly distributed, creating noise that maintains the topological character while being geometrically distinguishable from the signal.

\paragraph{Generation.} For each combination of topological object, noise type, and signal probability, we generate 10 independent datasets, each containing $n = 1000$ points. We sample signal probabilities at 50 equally spaced points from 0 to 1, providing detailed coverage of the transition from pure noise ($p_{signal}=0$) to pure signal ($p_{signal}=1$). This results in 500 datasets per topological object (50 probability values × 10 simulations each), totaling 2,000 datasets across the four noise conditions for spheres, plus 500 datasets each for torus and Klein bottle experiments. We use different random seeds for each simulation to ensure statistical independence while maintaining reproducibility.

This appendix presents complementary results across five additional topological scenarios examined in our study, demonstrating the generalizability of \method~beyond the 4D torus results presented in the main text. We explore Klein bottle and spheres with four types of noise: line, Laplace, cube, and plane configurations.

\subsection{Death scale discrimination}
Additional results across all topological scenarios (Figure \ref{fig:rq1_appendix}) demonstrate consistency in \method's discriminative power. The key patterns mirror those observed in the 4D torus: strong discrimination between signal and noise populations with log-rank $p$-values consistently below 0.05 across the signal probability spectrum, and characteristic decreasing median survival differences that transition from positive (signal features survive longer) at low signal probabilities to negative at high signal probabilities.
\begin{figure*}
    \centering
    \includegraphics[width=0.9\linewidth]{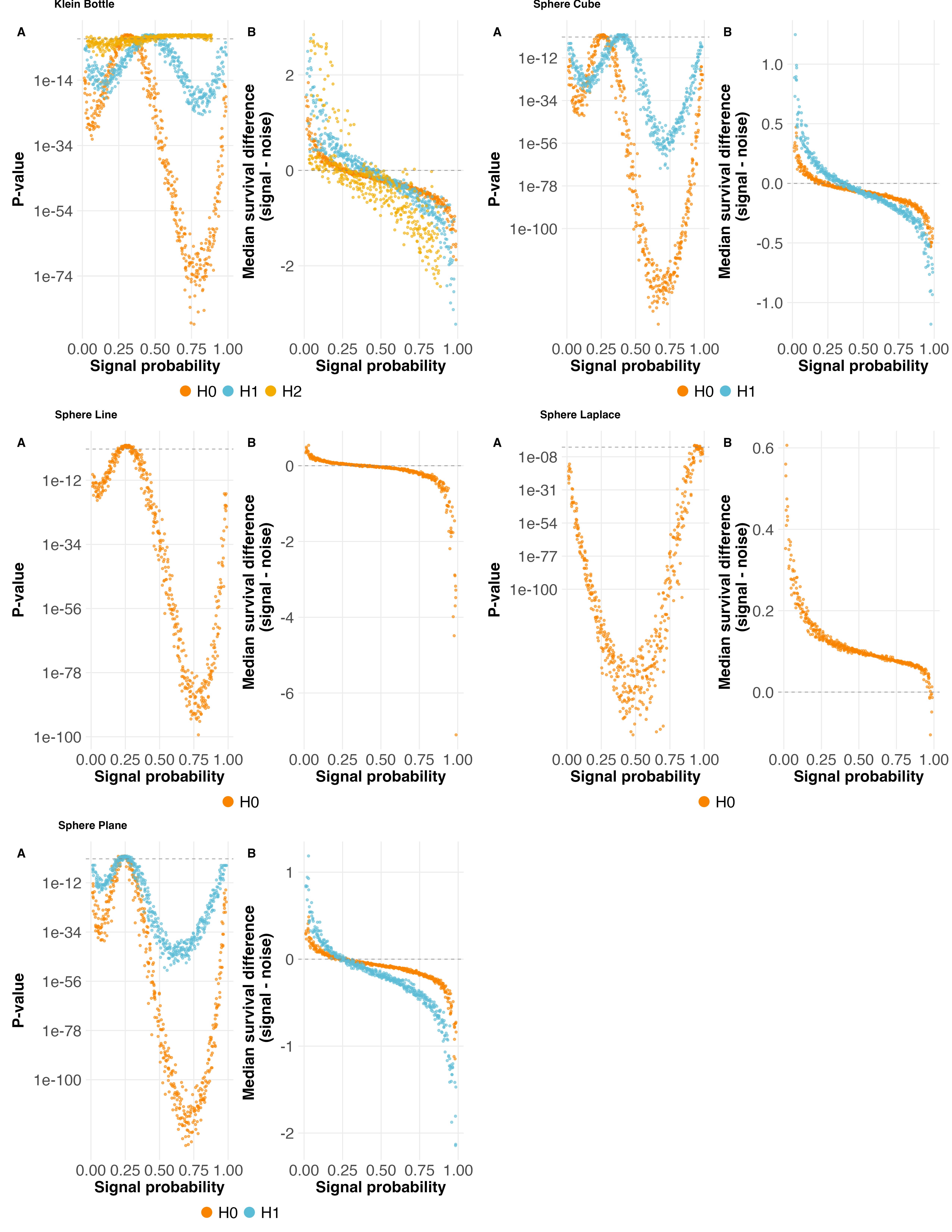}
    \caption{Signal vs. noise discrimination analysis across Klein bottle, and spheres with four type of noise (line, Laplace, cube, plane). Each point represents one simulation from 500 synthetic datasets spanning 50 signal probabilities in [0,1] with 10 independent realizations each. (A) shows log-rank test $p$-values comparing signal and noise survival distributions, and (B) shows median survival time differences (signal minus noise).}
    \label{fig:rq1_appendix}
\end{figure*}

\begin{table}[h]
\centering
\caption{Type~I error for parametric distributions ($n=500$ replications).}
\label{tab:type1_parametric}
\begin{tabular}{lccc}
\toprule
Distribution & Observed $\hat{\alpha}$ & 95\% CI & Status \\
\midrule
Exponential($\lambda=1$) & 0.062 & [0.044, 0.087] & \checkmark Valid \\
Weibull($k=2$) & 0.062 & [0.044, 0.087] & \checkmark Valid \\
LogNormal($\mu=0, \sigma=1$) & 0.044 & [0.029, 0.066] & \checkmark Valid \\
\bottomrule

    \end{tabular}
    \vskip -0.1in
\end{table}

\paragraph{Interpretation.} Signal-vs-signal tests show conservative Type~I error because pure manifold samples produce highly reproducible topological structure. This is expected behavior: when both groups are drawn from the same well-defined manifold, their persistence diagrams are nearly identical, making false rejections rare.

\section{Permutation calibration of the asymptotic log-rank test}\label{app:permutation_calibration}
 
The asymptotic log-rank test (Eq.~\ref{eq:logrank}) assumes exchangeability of pooled persistence values across diagrams. Within a single diagram, persistence values are not independent: in $H_0$ they are nested by the elder rule, and in $H_k$ for $k \geq 1$ they are coupled through homological constraints on births and deaths. Letting $\bar\rho$ denote the average within-diagram correlation between persistence values, the variance of $\sum_j (d_{Aj} - E_{Aj})$ deviates from its working-independence form by a factor approximately $1 + (n - 1)\bar\rho$. When $\bar\rho > 0$ the test is anti-conservative; when $\bar\rho < 0$ --- the regime relevant for elder-rule coupling, where one feature's survival comes at the expense of another's --- the test is conservative.
 
The direction of the effect is theoretically clear, but its empirical magnitude on real persistence diagrams is not. We validate it by comparison against an assumption-free reference test.
 
\paragraph{Setup.}
For each of the synthetic manifold scenarios used to assess Type~I error in Section~\ref{sec:hyp_test_results} (4D torus, 4D Klein bottle, 3D sphere with four noise geometries; both $H_0$ and $H_1$, yielding 16 conditions in total\footnote{Klein bottle has trivial $H_2$, the sphere has trivial $H_1$ in some noise configurations; we restrict to conditions where the dimension contains features.}), we generate $n = 200$ paired samples under both the null and alternative. For each replicate we compute:
\begin{itemize}[leftmargin=*]
    \item the asymptotic log-rank $p$-value $p_{\mathrm{LR}}$ from Eq.~\eqref{eq:logrank};
    \item a permutation $p$-value $p_{\mathrm{perm}}$ obtained by shuffling group labels at the diagram level and recomputing the log-rank statistic; we use $B = 1{,}000$ permutations.
\end{itemize}
Diagram-level permutation preserves within-diagram dependence exactly --- whatever correlation exists between persistence values within a diagram is carried through every permuted replicate --- and is therefore exchangeable under the null without any independence assumption.
 
\paragraph{Findings.}
Table~\ref{tab:permutation_calibration} reports per-condition decision concordance (fraction of replicates where $p_{\mathrm{LR}}$ and $p_{\mathrm{perm}}$ agree on reject/retain at $\alpha = 0.05$), Spearman rank correlation between $p_{\mathrm{LR}}$ and $p_{\mathrm{perm}}$ across replicates, and the empirical rejection rate of the asymptotic test. Across all 16 conditions, decision concordance is $\geq 99\%$ and Spearman $\rho \geq 0.99$. Where the two tests disagree, the asymptotic test rejects less often than the permutation test, consistent with the conservative direction predicted under negative within-diagram correlation. Empirical rejection rates of the asymptotic test are at or below the nominal $\alpha = 0.05$ in every condition.
 
\begin{table}[h]
\centering
\caption{Per-condition agreement between the asymptotic log-rank test and a diagram-level permutation test. Decision concordance: fraction of replicates in agreement on reject/retain at $\alpha = 0.05$. Spearman $\rho$: rank correlation between the two $p$-values across replicates. $\hat\alpha_{\mathrm{LR}}$: empirical rejection rate of the asymptotic test ($n = 200$ replications per condition). Conditions with no features in a given dimension are omitted.}
\label{tab:permutation_calibration}
\small
\begin{tabular}{llccc}
\toprule
Manifold & Dim & Concordance & Spearman $\rho$ & $\hat\alpha_{\mathrm{LR}}$ \\
\midrule
Torus (4D)             & $H_0$ & 1.000 & 0.998 & 0.025 \\
                       & $H_1$ & 0.995 & 0.996 & 0.075 \\
Klein bottle (4D)      & $H_0$ & 1.000 & 0.999 & 0.000 \\
                       & $H_1$ & 1.000 & 0.997 & 0.015 \\
Sphere, cube noise     & $H_0$ & 1.000 & 0.999 & 0.020 \\
                       & $H_1$ & 0.995 & 0.995 & 0.025 \\
Sphere, plane noise    & $H_0$ & 1.000 & 0.998 & 0.025 \\
                       & $H_1$ & 0.990 & 0.993 & 0.030 \\
Sphere, line noise     & $H_0$ & 1.000 & 0.999 & 0.020 \\
                       & $H_1$ & 1.000 & 0.997 & 0.020 \\
Sphere, Laplace noise  & $H_0$ & 1.000 & 0.999 & 0.025 \\
                       & $H_1$ & 0.995 & 0.996 & 0.030 \\
\midrule
\textbf{Aggregate (16 conditions)} & --- & $\geq 0.99$ & $\geq 0.99$ & $\leq 0.05$ \\
\bottomrule
\end{tabular}
\end{table}
 
\paragraph{Conclusion.}
Within-diagram dependence under elder-rule and homological coupling is conservative for the asymptotic log-rank test in practice. The agreement between the asymptotic test and a diagram-level permutation reference is high enough ($\geq 99\%$ decision agreement, Spearman $\rho \geq 0.99$) that we adopt the asymptotic test as the user-facing recommendation. We treat this validation as a one-time check rather than a pipeline component: practitioners deploying \method~do not need to run the permutation reference per analysis. Diagram-level permutation remains available as a fallback for settings where exchangeability cannot be presumed (e.g., diagrams with strongly hierarchical filtrations).

\section{Power Analysis: Additional Manifolds}
\label{app:power}
We analyze \method's power to detect differences between signal (manifold) and noise features using simulated torus data with varying signal probability $p \in [0.02, 0.98]$. Figure~\ref{fig:power} shows empirical power as a function of hazard ratio for $H_0$ (connected components) and $H_1$ (loops) features.

\begin{figure}
    \centering
    \includegraphics[width=1\linewidth]{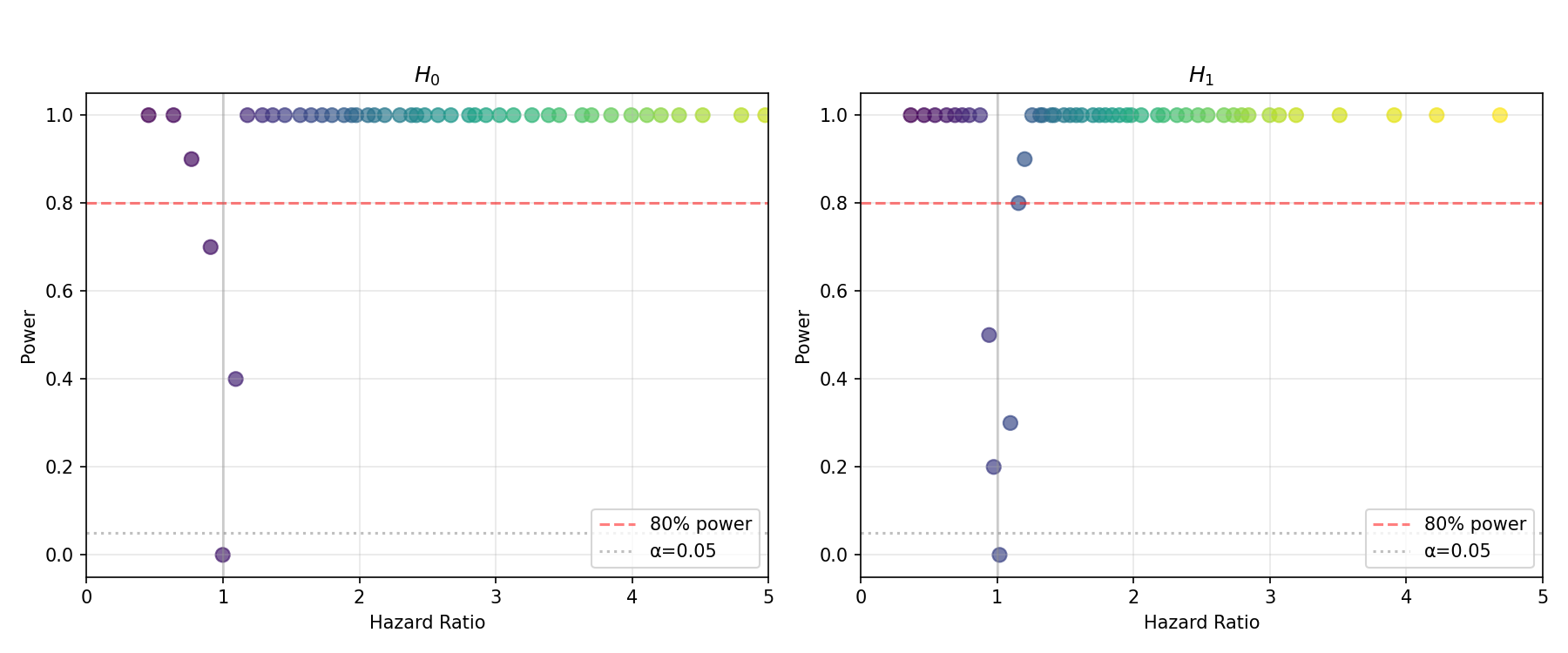}
    \caption{Power vs hazard ratio for \method~two-sample test on simulated torus data. Left: $H_0$ features. Right: $H_1$ features. Color indicates signal probability. The V-shaped curve confirms test validity: power $\approx 0$ at HR $= 1$ (no true difference) and power $\to 1$ as $|\log(\text{HR})|$ increases.}
    \label{fig:power}
\end{figure}

\begin{table}
    \centering
    \caption{Empirical power at key hazard ratios (Torus, $n=200$ replications per condition).}
    \label{tab:power_summary}
    \begin{tabular}{ccc}
        \toprule
        Hazard Ratio & $H_0$ Power & $H_1$ Power \\
        \midrule
        0.5 & 100\% & 100\% \\
        0.67 & 100\% & 100\% \\
        1.0 & 0\% & 0\% \\
        1.5 & 100\% & 100\% \\
        2.0 & 100\% & 100\% \\
        \bottomrule   

    \end{tabular}
    \vskip -0.1in
\end{table}

Key findings include: (1) power exceeds 80\% when HR deviates substantially from 1.0; (2) the V-shaped power curve confirms test validity, power $\approx 0\%$ at HR $= 1.0$ and approaches 100\% as $|\log(\text{HR})|$ increases; (3) $H_0$ features show slightly higher power due to larger sample sizes (more connected components than loops).

\section{Benchmark}\label{app:bench_results}
\subsection{Datasets}
Table \ref{tab:datasets} summarizes dataset characteristics. This benchmark selection follows established practice in topological machine learning~\cite{carriere2017sliced,adamsPersistenceImagesStable2017} and provides diversity across: (i)~graph sizes (14--284 average nodes), (ii)~classification difficulty (binary and multi-class), (iii)~domain semantics (molecular structure vs.\ social dynamics), and (iv)~class balance characteristics. The molecular datasets test whether \method~captures chemically meaningful topological features, while social network datasets assess generalization to domains where filtration semantics differ.

\begin{table}[t]
\centering
\caption{Dataset statistics. Molecular datasets (top) represent chemical compounds; social network datasets (middle) represent user interactions; 3D point cloud datasets (bottom) are sampled from surface meshes.}
\label{tab:datasets}
\small
\begin{tabular}{lrrl}
\toprule
Dataset & Samples & Classes & Domain \\
\midrule
MUTAG       & 188   & 2 & Mutagenicity \\
PTC-MR      & 344   & 2 & Toxicity \\
PROTEINS    & 1{,}113  & 2 & Protein structure \\
ENZYMES     & 600   & 6 & Enzyme function \\
NCI1        & 4{,}110  & 2 & Anti-cancer activity \\
DHFR        & 467   & 2 & Drug activity \\
COX2        & 467   & 2 & Drug activity \\
DD          & 1{,}178  & 2 & Protein structure \\
\midrule
IMDB-B      & 1{,}000  & 2 & Movie collaboration \\
IMDB-M      & 1{,}500  & 3 & Movie collaboration \\
\midrule
SwissCheese & 1{,}000  & 2   & Synthetic 3D point clouds \\
Corals      & 83       & 2   & 3D coral surface meshes \\
MCB         & 2{,}441  & 21  & CAD mechanical components \\
ModelNet10  & 4{,}899  & 10  & CAD object meshes \\
\bottomrule
\end{tabular}
\vskip -0.1in
\end{table}

\subsection{Hyperparameter selection}\label{app:hyperparams}

For all baselines, hyperparameters were selected by grid search under the common
cross-validation protocol of Section~\ref{sec:vectorisation_results} (stratified
10-fold for graphs; the 10 predefined splits of \citet{graf2025flood} for point
clouds), with selection performed on the training portion of each split only. Table~\ref{tab:hyperparam_grids} lists the search grids.

\begin{table}[h]
\centering
\caption{Hyperparameter search grids.}
\label{tab:hyperparam_grids}
\small
\begin{tabular}{ll}
\toprule
Method / Component & Search grid \\
\midrule
\multicolumn{2}{@{}l}{\textit{Persistence kernels}} \\
SWK  & bandwidth $\in \{0.1, 1, 10\}$; \# directions $= 10$ \\
PSSK & bandwidth $\sigma \in \{0.1, 1, 10\}$ \\
\midrule
\multicolumn{2}{@{}l}{\textit{Classifiers}} \\
SVM (RBF, vectorisations) & $C \in \{10^{-2}, 10^{-1}, 1, 10, 10^{2}\}$; $\gamma \in \{\text{scale}, 10^{-3}, 10^{-2}, 10^{-1}\}$ \\
SVM (precomputed, kernels) & $C \in \{10^{-2}, 10^{-1}, 1, 10, 10^{2}\}$ \\
Random Forest & \# trees $\in \{100, 300, 500\}$; max depth $\in \{\text{None}, 10, 20\}$; \\
              & \quad max features $\in \{\sqrt{\cdot},\, \log_2\}$ \\
\bottomrule
\end{tabular}
\end{table}

\subsection{Results}
Tables~\ref{tab:appendix_accuracy_svm}--\ref{tab:appendix_mcc_rf} present complete results for all \method~variants and baselines across all metrics and classifiers. Values are reported as mean $\pm$ standard deviation over 10-fold cross-validation.

\begin{table*}[t]
\centering
\caption{Accuracy (\%) for all methods (SVM). S=\method, B=\method-B (birth entry), F=\method-F (birth features). Numbers indicate grid size. Bold indicates best per row.}
\label{tab:appendix_accuracy_svm}
\tiny
\begin{tabular}{lcccccccccccc}
\toprule
Dataset & S25 & S50 & S100 & B25 & B50 & B100 & F25 & F50 & F100 & PI & PL & BC \\
\midrule
COX2 & 78.2±0.8 & 78.2±0.8 & 78.2±0.8 & 78.2±0.8 & 78.2±0.8 & 78.2±0.8 & 78.5±1.1 & 78.5±1.1 & 78.5±1.1 & 78.7±1.2 & 78.2±0.8 & \textbf{78.8}±1.4 \\
DD & 70.4±2.5 & 70.0±2.7 & 69.6±2.5 & 71.3±2.8 & 69.9±2.6 & 69.4±2.5 & 71.1±2.8 & 70.9±2.5 & 70.3±3.1 & 75.0±3.2 & 72.1±3.5 & \textbf{75.4}±3.8 \\
DHFR & 64.7±4.4 & 64.1±4.6 & 62.8±4.5 & 60.8±1.2 & 60.8±0.6 & 61.0±0.5 & \textbf{67.6}±4.4 & 66.5±5.0 & 66.8±4.9 & 64.3±4.6 & 61.0±0.5 & 62.6±3.2 \\
ENZYMES & 25.6±4.9 & 24.9±4.6 & 24.1±4.0 & 22.6±4.2 & 21.7±4.4 & 21.9±4.1 & \textbf{27.0}±4.2 & 26.8±3.4 & 26.4±3.8 & 24.3±4.2 & 25.2±4.9 & 22.0±3.7 \\
IMDB-BINARY & 63.6±3.7 & 63.6±3.7 & 64.0±3.8 & 61.2±4.8 & 61.1±4.8 & 60.8±5.1 & 63.7±4.2 & 63.8±4.1 & \textbf{64.2}±4.2 & 62.5±4.0 & 63.7±4.0 & 63.4±3.8 \\
IMDB-MULTI & 41.0±3.2 & \textbf{41.3}±3.5 & 41.2±2.9 & 39.6±2.9 & 39.6±2.9 & 39.6±2.9 & 40.5±3.6 & 40.7±3.2 & 40.5±3.2 & 40.3±4.0 & 39.2±2.7 & 38.4±2.8 \\
MUTAG & 77.9±8.7 & 74.9±9.1 & 75.8±7.7 & 71.3±8.6 & 69.0±5.7 & 67.2±3.8 & 74.9±7.9 & 75.8±7.7 & 74.0±6.9 & \textbf{78.8}±8.2 & 72.2±7.5 & 78.4±8.3 \\
NCI1 & 62.3±2.0 & 62.2±2.0 & 62.3±1.8 & 62.4±2.0 & 62.6±2.1 & 62.4±2.1 & 62.5±1.9 & 62.4±2.0 & 61.9±1.9 & \textbf{63.0}±2.0 & 60.4±2.5 & 62.0±1.8 \\
PROTEINS & 66.4±3.6 & 65.7±3.8 & 65.8±3.7 & 65.8±3.2 & 65.7±3.1 & 65.7±3.1 & 66.4±3.5 & 66.1±3.5 & 66.2±3.2 & \textbf{71.0}±4.2 & 70.6±4.4 & 70.0±4.5 \\
PTC-MR & 56.8±5.5 & 56.9±6.0 & 55.9±5.4 & \textbf{58.5}±2.8 & 57.6±2.4 & 57.6±2.4 & 57.0±5.6 & 56.0±5.2 & 56.0±5.2 & 55.3±6.5 & 54.7±8.0 & 53.6±6.0 \\
\midrule
\textbf{Avg} & 60.7 & 60.2 & 60.0 & 59.2 & 58.6 & 58.4 & 60.9 & 60.7 & 60.5 & \textbf{61.3} & 59.7 & 60.5 \\
\bottomrule
\end{tabular}
\end{table*}

\begin{table*}[t]
\centering
\caption{Accuracy (\%) for all methods (RF). S=\method, B=\method-B (birth entry), F=\method-F (birth features). Numbers indicate grid size. Bold indicates best per row.}
\label{tab:appendix_accuracy_rf}
\tiny
\begin{tabular}{lcccccccccccc}
\toprule
Dataset & S25 & S50 & S100 & B25 & B50 & B100 & F25 & F50 & F100 & PI & PL & BC \\
\midrule
COX2 & 76.8±3.8 & 77.0±4.2 & 76.5±4.1 & 76.2±3.1 & 76.2±3.1 & 75.9±3.0 & 77.5±3.7 & 76.9±3.6 & 77.1±3.7 & 77.5±3.9 & \textbf{78.2}±0.8 & 77.9±4.2 \\
DD & 73.3±2.4 & 73.1±2.7 & 71.8±3.1 & 73.3±3.3 & 73.1±3.5 & 72.8±3.3 & 74.1±2.4 & 73.1±3.0 & 73.1±3.3 & \textbf{75.2}±3.5 & 72.3±3.2 & 74.9±3.7 \\
DHFR & 71.0±4.0 & 70.9±4.3 & 71.1±4.0 & 66.6±5.0 & 66.2±4.7 & 66.0±4.6 & 72.0±4.0 & 72.0±3.9 & 72.1±4.0 & \textbf{72.5}±4.3 & 61.0±0.5 & 71.3±4.2 \\
ENZYMES & 23.4±5.4 & 22.4±4.6 & 21.8±4.9 & 23.8±4.4 & 23.9±4.6 & 23.1±4.4 & 27.6±4.4 & 27.6±4.6 & 27.7±5.1 & 27.3±4.8 & \textbf{27.8}±6.0 & 27.6±5.4 \\
IMDB-BINARY & 62.5±4.1 & 63.8±3.5 & 62.8±3.0 & 64.7±3.4 & \textbf{65.3}±3.7 & 65.0±3.0 & 64.9±4.4 & 65.1±4.4 & 64.9±4.2 & 64.1±4.8 & 61.7±5.0 & 61.4±4.1 \\
IMDB-MULTI & 38.9±3.1 & 40.2±3.1 & \textbf{41.8}±3.3 & 38.0±2.8 & 38.3±2.9 & 38.0±2.6 & 41.3±3.0 & 41.2±3.0 & 41.4±2.9 & 41.2±3.4 & 41.5±2.8 & 39.9±2.7 \\
MUTAG & 77.7±8.5 & 77.7±8.5 & 77.7±8.5 & 70.8±9.6 & 70.8±9.6 & 70.8±9.6 & 77.7±8.5 & 77.7±8.5 & 77.7±8.5 & \textbf{77.9}±8.4 & 73.8±7.7 & 76.6±8.6 \\
NCI1 & 61.7±2.1 & 61.5±2.1 & 61.4±2.2 & 62.6±2.3 & 62.6±2.4 & 62.6±2.3 & 62.1±2.1 & 62.1±2.1 & 61.8±1.9 & \textbf{63.1}±2.1 & 61.0±2.5 & 62.0±1.9 \\
PROTEINS & 68.3±5.1 & 67.9±5.1 & 67.6±5.0 & 67.5±4.3 & 67.4±4.5 & 67.3±4.6 & 68.8±4.4 & 68.7±4.5 & 68.4±4.3 & \textbf{70.7}±4.0 & 68.9±3.8 & 70.4±4.4 \\
PTC-MR & 51.6±6.7 & 51.6±6.7 & 51.6±6.9 & \textbf{58.1}±7.8 & \textbf{58.1}±7.8 & 57.9±7.7 & 51.7±7.6 & 51.9±7.3 & 51.8±7.2 & 54.2±7.0 & 54.7±8.0 & 54.4±7.6 \\
\midrule
\textbf{Avg} & 60.5 & 60.6 & 60.4 & 60.2 & 60.2 & 59.9 & 61.8 & 61.6 & 61.6 & \textbf{62.4} & 60.1 & 61.6 \\
\bottomrule
\end{tabular}
\end{table*}

\begin{table*}[t]
\centering
\caption{Bal. Acc. (\%) for all methods (SVM). S=\method, B=\method-B (birth entry), F=\method-F (birth features). Numbers indicate grid size. Bold indicates best per row.}
\label{tab:appendix_balanced_accuracy_svm}
\tiny
\begin{tabular}{lcccccccccccc}
\toprule
Dataset & S25 & S50 & S100 & B25 & B50 & B100 & F25 & F50 & F100 & PI & PL & BC \\
\midrule
COX2 & 50.0±0.0 & 50.0±0.0 & 50.0±0.0 & 50.0±0.0 & 50.0±0.0 & 50.0±0.0 & 50.8±1.8 & 50.8±1.8 & 50.8±1.8 & \textbf{51.5}±2.5 & 50.0±0.0 & \textbf{51.5}±2.6 \\
DD & 67.2±2.9 & 66.9±3.2 & 66.4±2.9 & 68.1±3.2 & 66.4±2.8 & 65.7±2.8 & 68.0±3.2 & 67.8±2.8 & 66.9±3.4 & 73.1±3.2 & 68.9±3.8 & \textbf{73.3}±3.9 \\
DHFR & 60.4±5.1 & 60.2±5.2 & 58.9±5.3 & 50.1±1.2 & 49.9±0.4 & 50.0±0.0 & \textbf{65.5}±4.6 & 64.5±5.3 & 65.4±5.0 & 60.6±5.1 & 50.0±0.0 & 54.4±3.4 \\
ENZYMES & 25.6±4.9 & 24.9±4.6 & 24.1±4.0 & 22.6±4.2 & 21.7±4.4 & 21.9±4.1 & \textbf{27.0}±4.2 & 26.8±3.4 & 26.4±3.8 & 24.3±4.2 & 25.2±4.9 & 22.0±3.7 \\
IMDB-BINARY & 63.6±3.7 & 63.6±3.7 & 64.0±3.8 & 61.2±4.8 & 61.1±4.8 & 60.8±5.1 & 63.7±4.2 & 63.8±4.1 & \textbf{64.2}±4.2 & 62.5±4.0 & 63.7±4.0 & 63.4±3.8 \\
IMDB-MULTI & 41.0±3.2 & \textbf{41.3}±3.5 & 41.2±2.9 & 39.6±2.9 & 39.6±2.9 & 39.6±2.9 & 40.5±3.6 & 40.7±3.2 & 40.5±3.2 & 40.3±4.0 & 39.2±2.7 & 38.4±2.8 \\
MUTAG & 71.0±10.3 & 67.9±10.8 & 67.4±8.9 & 62.8±9.5 & 55.3±7.0 & 51.4±4.8 & 67.4±9.1 & 67.4±8.9 & 63.8±8.6 & \textbf{75.2}±10.0 & 62.3±8.4 & 72.7±10.5 \\
NCI1 & 62.3±2.0 & 62.3±2.0 & 62.3±1.8 & 62.4±2.0 & 62.6±2.1 & 62.4±2.1 & 62.5±1.9 & 62.4±1.9 & 62.0±1.9 & \textbf{63.1}±2.0 & 60.4±2.5 & 62.1±1.7 \\
PROTEINS & 60.8±4.1 & 60.0±4.2 & 60.0±4.2 & 59.0±3.7 & 58.9±3.7 & 58.9±3.7 & 60.4±4.1 & 60.0±4.0 & 60.1±3.8 & \textbf{68.5}±4.4 & 67.4±4.4 & 67.1±4.8 \\
PTC-MR & 53.7±5.8 & 53.7±6.3 & 52.5±5.4 & 53.1±2.8 & 52.0±2.4 & 52.0±2.4 & \textbf{53.9}±5.9 & 52.8±5.4 & 52.8±5.4 & 52.4±5.9 & 52.6±7.5 & 51.0±5.9 \\
\midrule
\textbf{Avg} & 55.6 & 55.1 & 54.7 & 52.9 & 51.8 & 51.3 & 56.0 & 55.7 & 55.3 & \textbf{57.1} & 54.0 & 55.6 \\
\bottomrule
\end{tabular}
\end{table*}

\begin{table*}[t]
\centering
\caption{Bal. Acc. (\%) for all methods (RF). S=\method, B=\method-B (birth entry), F=\method-F (birth features). Numbers indicate grid size. Bold indicates best per row.}
\label{tab:appendix_balanced_accuracy_rf}
\tiny
\begin{tabular}{lcccccccccccc}
\toprule
Dataset & S25 & S50 & S100 & B25 & B50 & B100 & F25 & F50 & F100 & PI & PL & BC \\
\midrule
COX2 & 56.4±5.4 & 56.8±6.2 & 56.5±6.1 & 50.9±3.5 & 50.9±3.3 & 50.3±3.0 & 57.3±5.5 & 56.6±5.1 & 57.2±5.6 & 57.5±5.3 & 50.0±0.0 & \textbf{57.5}±6.5 \\
DD & 71.4±2.8 & 71.3±3.0 & 69.8±3.5 & 71.4±3.5 & 71.0±3.7 & 70.8±3.5 & 72.3±2.5 & 71.1±3.3 & 71.1±3.5 & \textbf{73.8}±3.6 & 70.0±3.5 & 72.9±3.9 \\
DHFR & 68.8±3.9 & 68.6±4.2 & 69.0±3.9 & 64.0±5.2 & 63.7±4.9 & 63.5±4.7 & 70.1±4.0 & 70.0±3.8 & 70.1±3.7 & \textbf{70.4}±4.5 & 50.0±0.0 & 69.2±4.5 \\
ENZYMES & 23.4±5.4 & 22.4±4.6 & 21.8±4.9 & 23.8±4.4 & 23.9±4.6 & 23.1±4.4 & 27.6±4.4 & 27.6±4.6 & 27.7±5.1 & 27.3±4.8 & \textbf{27.8}±6.0 & 27.6±5.4 \\
IMDB-BINARY & 62.5±4.1 & 63.8±3.5 & 62.8±3.0 & 64.7±3.4 & \textbf{65.3}±3.7 & 65.0±3.0 & 64.9±4.4 & 65.1±4.4 & 64.9±4.2 & 64.1±4.8 & 61.7±5.0 & 61.4±4.1 \\
IMDB-MULTI & 38.9±3.1 & 40.2±3.1 & \textbf{41.8}±3.3 & 38.0±2.8 & 38.3±2.9 & 38.0±2.6 & 41.3±3.0 & 41.2±3.0 & 41.4±2.9 & 41.2±3.4 & 41.5±2.8 & 39.9±2.7 \\
MUTAG & 74.9±9.5 & 74.9±9.5 & \textbf{75.0}±9.5 & 63.6±10.3 & 63.6±10.3 & 63.6±10.3 & 74.9±9.5 & 74.9±9.5 & 74.9±9.5 & 74.8±9.7 & 65.2±8.9 & 73.6±9.6 \\
NCI1 & 61.7±2.1 & 61.6±2.1 & 61.4±2.2 & 62.6±2.3 & 62.6±2.4 & 62.6±2.3 & 62.1±2.1 & 62.1±2.1 & 61.8±1.9 & \textbf{63.1}±2.1 & 61.0±2.5 & 62.0±1.9 \\
PROTEINS & 66.1±5.2 & 65.7±5.1 & 65.4±5.1 & 65.3±4.2 & 65.3±4.4 & 65.2±4.6 & 66.7±4.6 & 66.6±4.7 & 66.4±4.6 & \textbf{68.7}±4.0 & 66.9±3.9 & 68.5±4.6 \\
PTC-MR & 50.7±6.4 & 50.6±6.4 & 50.7±6.6 & \textbf{56.4}±7.8 & 56.4±7.7 & 56.1±7.6 & 50.7±7.2 & 50.9±7.0 & 50.8±6.9 & 53.5±6.7 & 53.6±7.9 & 53.8±7.1 \\
\midrule
\textbf{Avg} & 57.5 & 57.6 & 57.4 & 56.1 & 56.1 & 55.8 & 58.8 & 58.6 & 58.6 & \textbf{59.4} & 54.8 & 58.6 \\
\bottomrule
\end{tabular}
\end{table*}

\begin{table*}[t]
\centering
\caption{F1 (\%) for all methods (SVM). S=\method, B=\method-B (birth entry), F=\method-F (birth features). Numbers indicate grid size. Bold indicates best per row.}
\label{tab:appendix_f1_svm}
\tiny
\begin{tabular}{lcccccccccccc}
\toprule
Dataset & S25 & S50 & S100 & B25 & B50 & B100 & F25 & F50 & F100 & PI & PL & BC \\
\midrule
COX2 & 68.6±1.1 & 68.6±1.1 & 68.6±1.1 & 68.6±1.1 & 68.6±1.1 & 68.6±1.1 & 69.4±2.1 & 69.4±2.1 & 69.4±2.1 & \textbf{70.0}±2.4 & 68.6±1.1 & 70.0±2.7 \\
DD & 69.1±3.0 & 68.7±3.3 & 68.3±3.0 & 70.0±3.2 & 68.3±2.9 & 67.6±2.9 & 69.9±3.2 & 69.7±2.8 & 68.8±3.5 & 74.6±3.2 & 70.8±3.8 & \textbf{74.9}±3.9 \\
DHFR & 63.0±4.9 & 62.7±5.0 & 61.4±5.0 & 46.9±1.9 & 46.1±0.5 & 46.2±0.6 & \textbf{67.3}±4.4 & 66.2±5.1 & 66.7±4.9 & 63.0±5.1 & 46.2±0.6 & 55.8±4.0 \\
ENZYMES & 24.7±5.4 & 23.5±5.2 & 22.5±4.7 & 21.5±4.2 & 20.4±4.6 & 20.7±4.1 & \textbf{26.6}±4.5 & 26.3±3.7 & 26.0±4.0 & 23.0±4.5 & 22.7±5.1 & 20.1±4.3 \\
IMDB-BINARY & 62.6±3.8 & 62.6±3.8 & 62.8±4.0 & 60.9±4.7 & 60.8±4.7 & 60.6±5.0 & 62.9±4.2 & 62.9±4.1 & \textbf{63.3}±4.2 & 61.8±4.0 & 62.8±4.1 & 62.5±3.8 \\
IMDB-MULTI & 37.3±3.7 & \textbf{37.8}±4.0 & 37.7±3.3 & 31.5±2.4 & 31.5±2.4 & 31.5±2.4 & 37.0±4.2 & 37.5±3.7 & 37.1±3.6 & 37.2±4.7 & 34.4±3.5 & 32.8±3.8 \\
MUTAG & 76.1±10.0 & 72.9±10.2 & 73.1±9.0 & 68.6±9.6 & 60.1±8.1 & 54.7±6.6 & 72.8±8.8 & 73.1±9.0 & 69.9±9.0 & \textbf{78.1}±8.7 & 67.5±9.3 & 76.8±9.5 \\
NCI1 & 61.5±2.1 & 61.4±2.0 & 61.5±1.9 & 61.3±2.1 & 61.6±2.3 & 61.1±2.2 & 61.4±2.1 & 61.3±2.2 & 61.1±2.0 & \textbf{62.1}±2.2 & 60.0±2.6 & 61.4±1.9 \\
PROTEINS & 62.7±4.6 & 61.7±4.7 & 61.7±4.7 & 59.9±4.5 & 59.9±4.5 & 59.9±4.5 & 62.0±4.7 & 61.5±4.7 & 61.6±4.4 & \textbf{70.4}±4.3 & 69.6±4.4 & 69.1±4.7 \\
PTC-MR & 52.3±7.2 & 52.2±8.0 & 50.7±7.0 & 45.6±5.3 & 43.6±4.6 & 43.6±4.6 & \textbf{52.6}±7.5 & 51.3±7.1 & 51.2±7.0 & 51.3±6.1 & 52.5±7.4 & 50.4±6.4 \\
\midrule
\textbf{Avg} & 57.8 & 57.2 & 56.8 & 53.5 & 52.1 & 51.4 & 58.2 & 57.9 & 57.5 & \textbf{59.1} & 55.5 & 57.4 \\
\bottomrule
\end{tabular}
\end{table*}

\begin{table*}[t]
\centering
\caption{F1 (\%) for all methods (RF). S=\method, B=\method-B (birth entry), F=\method-F (birth features). Numbers indicate grid size. Bold indicates best per row.}
\label{tab:appendix_f1_rf}
\tiny
\begin{tabular}{lcccccccccccc}
\toprule
Dataset & S25 & S50 & S100 & B25 & B50 & B100 & F25 & F50 & F100 & PI & PL & BC \\
\midrule
COX2 & 73.2±4.1 & 73.3±4.7 & 73.0±4.6 & 69.5±2.9 & 69.5±2.9 & 69.0±2.6 & 73.9±3.9 & 73.3±3.7 & 73.6±4.0 & 74.0±4.0 & 68.6±1.1 & \textbf{74.2}±4.8 \\
DD & 72.9±2.6 & 72.7±2.8 & 71.3±3.3 & 72.9±3.4 & 72.6±3.7 & 72.3±3.4 & 73.7±2.4 & 72.6±3.1 & 72.7±3.4 & \textbf{75.0}±3.5 & 71.7±3.4 & 74.4±3.8 \\
DHFR & 70.7±3.9 & 70.6±4.2 & 70.8±3.9 & 66.1±5.1 & 65.8±4.8 & 65.6±4.6 & 71.8±4.0 & 71.8±3.8 & 71.9±3.8 & \textbf{72.2}±4.3 & 46.2±0.6 & 71.0±4.3 \\
ENZYMES & 23.1±5.6 & 22.0±4.6 & 21.5±4.9 & 23.3±4.5 & 23.2±4.6 & 22.4±4.4 & 27.1±4.4 & 27.1±4.7 & 27.3±5.2 & 26.8±5.0 & \textbf{27.3}±6.1 & 27.1±5.4 \\
IMDB-BINARY & 62.3±4.1 & 63.7±3.6 & 62.7±3.0 & 64.0±3.5 & 64.6±3.9 & 64.5±3.2 & 64.8±4.4 & \textbf{65.0}±4.4 & 64.7±4.3 & 64.0±4.8 & 61.6±5.0 & 61.3±4.0 \\
IMDB-MULTI & 35.6±3.4 & 36.7±3.4 & \textbf{38.6}±3.8 & 34.1±3.2 & 34.2±3.2 & 34.0±3.0 & 37.9±3.6 & 37.9±3.4 & 38.1±3.4 & 37.8±3.9 & 38.2±3.3 & 36.2±2.8 \\
MUTAG & 77.4±8.4 & 77.4±8.4 & \textbf{77.4}±8.4 & 68.4±10.5 & 68.4±10.5 & 68.4±10.5 & 77.4±8.4 & 77.4±8.4 & 77.4±8.4 & \textbf{77.4}±8.5 & 70.3±9.1 & 76.2±8.6 \\
NCI1 & 61.5±2.0 & 61.4±2.0 & 61.2±2.1 & 62.2±2.4 & 62.3±2.4 & 62.3±2.4 & 61.9±2.0 & 61.8±2.1 & 61.5±1.9 & \textbf{62.9}±2.1 & 60.6±2.5 & 61.7±1.9 \\
PROTEINS & 67.8±5.2 & 67.4±5.1 & 67.1±5.1 & 67.1±4.2 & 67.0±4.4 & 66.9±4.6 & 68.4±4.4 & 68.3±4.6 & 68.0±4.4 & \textbf{70.3}±3.9 & 68.5±3.8 & 70.0±4.4 \\
PTC-MR & 50.9±6.5 & 50.9±6.4 & 50.9±6.6 & \textbf{56.6}±8.0 & 56.6±7.9 & 56.3±7.8 & 51.0±7.3 & 51.1±7.1 & 51.1±7.0 & 53.5±6.9 & 53.8±7.7 & 53.8±7.5 \\
\midrule
\textbf{Avg} & 59.5 & 59.6 & 59.5 & 58.4 & 58.4 & 58.2 & 60.8 & 60.6 & 60.6 & \textbf{61.4} & 56.7 & 60.6 \\
\bottomrule
\end{tabular}
\end{table*}

\begin{table*}[t]
\centering
\caption{MCC (\%) for all methods (SVM). S=\method, B=\method-B (birth entry), F=\method-F (birth features). Numbers indicate grid size. Bold indicates best per row.}
\label{tab:appendix_mcc_svm}
\tiny
\begin{tabular}{lcccccccccccc}
\toprule
Dataset & S25 & S50 & S100 & B25 & B50 & B100 & F25 & F50 & F100 & PI & PL & BC \\
\midrule
COX2 & 0.00±0.00 & 0.00±0.00 & 0.00±0.00 & 0.00±0.00 & 0.00±0.00 & 0.00±0.00 & 0.05±0.10 & 0.05±0.10 & 0.05±0.10 & 0.08±0.13 & 0.00±0.00 & \textbf{0.08}±0.13 \\
DD & 0.38±0.06 & 0.37±0.06 & 0.36±0.06 & 0.40±0.06 & 0.37±0.06 & 0.35±0.06 & 0.39±0.06 & 0.39±0.06 & 0.37±0.07 & 0.48±0.07 & 0.41±0.08 & \textbf{0.49}±0.08 \\
DHFR & 0.23±0.10 & 0.22±0.11 & 0.19±0.11 & 0.01±0.08 & -0.01±0.03 & 0.00±0.00 & \textbf{0.32}±0.09 & 0.29±0.11 & 0.31±0.10 & 0.22±0.11 & 0.00±0.00 & 0.13±0.10 \\
ENZYMES & 0.11±0.06 & 0.10±0.06 & 0.09±0.05 & 0.07±0.05 & 0.06±0.05 & 0.06±0.05 & \textbf{0.13}±0.05 & 0.12±0.04 & 0.12±0.05 & 0.09±0.05 & 0.11±0.06 & 0.07±0.05 \\
IMDB-BINARY & 0.29±0.08 & 0.29±0.08 & 0.30±0.08 & 0.23±0.10 & 0.23±0.10 & 0.22±0.10 & 0.29±0.09 & 0.29±0.09 & \textbf{0.30}±0.09 & 0.26±0.08 & 0.29±0.09 & 0.28±0.08 \\
IMDB-MULTI & 0.13±0.05 & \textbf{0.13}±0.06 & 0.13±0.05 & 0.11±0.05 & 0.11±0.05 & 0.11±0.05 & 0.12±0.06 & 0.12±0.05 & 0.12±0.05 & 0.11±0.07 & 0.10±0.04 & 0.09±0.05 \\
MUTAG & 0.48±0.22 & 0.41±0.23 & 0.43±0.20 & 0.31±0.23 & 0.17±0.20 & 0.04±0.12 & 0.41±0.20 & 0.43±0.20 & 0.36±0.21 & \textbf{0.53}±0.19 & 0.33±0.19 & 0.50±0.20 \\
NCI1 & 0.26±0.04 & 0.26±0.04 & 0.26±0.04 & 0.26±0.04 & 0.27±0.04 & 0.27±0.04 & 0.26±0.04 & 0.26±0.04 & 0.25±0.04 & \textbf{0.28}±0.04 & 0.21±0.05 & 0.25±0.04 \\
PROTEINS & 0.27±0.10 & 0.25±0.10 & 0.26±0.10 & 0.26±0.09 & 0.26±0.09 & 0.26±0.09 & 0.27±0.10 & 0.26±0.09 & 0.27±0.09 & \textbf{0.39}±0.09 & 0.38±0.10 & 0.36±0.10 \\
PTC-MR & 0.09±0.14 & 0.08±0.15 & 0.06±0.15 & \textbf{0.15}±0.12 & 0.10±0.11 & 0.10±0.11 & 0.09±0.15 & 0.06±0.14 & 0.06±0.14 & 0.07±0.14 & 0.07±0.17 & 0.03±0.14 \\
\midrule
\textbf{Avg} & 0.22 & 0.21 & 0.21 & 0.18 & 0.16 & 0.14 & 0.23 & 0.23 & 0.22 & \textbf{0.25} & 0.19 & 0.23 \\
\bottomrule
\end{tabular}
\end{table*}

\begin{table*}[t]
\centering
\caption{MCC (\%) for all methods (RF). S=\method, B=\method-B (birth entry), F=\method-F (birth features). Numbers indicate grid size. Bold indicates best per row.}
\label{tab:appendix_mcc_rf}
\tiny
\begin{tabular}{lcccccccccccc}
\toprule
Dataset & S25 & S50 & S100 & B25 & B50 & B100 & F25 & F50 & F100 & PI & PL & BC \\
\midrule
COX2 & 0.18±0.15 & 0.18±0.16 & 0.17±0.16 & 0.04±0.15 & 0.05±0.15 & 0.02±0.13 & 0.21±0.14 & 0.18±0.13 & 0.19±0.14 & \textbf{0.21}±0.15 & 0.00±0.00 & 0.21±0.17 \\
DD & 0.44±0.05 & 0.44±0.06 & 0.41±0.07 & 0.44±0.07 & 0.44±0.08 & 0.43±0.07 & 0.46±0.05 & 0.44±0.07 & 0.44±0.07 & \textbf{0.48}±0.07 & 0.42±0.07 & 0.48±0.08 \\
DHFR & 0.39±0.08 & 0.38±0.09 & 0.39±0.08 & 0.29±0.11 & 0.28±0.10 & 0.28±0.10 & 0.41±0.08 & 0.41±0.08 & 0.41±0.08 & \textbf{0.42}±0.09 & 0.00±0.00 & 0.39±0.09 \\
ENZYMES & 0.08±0.07 & 0.07±0.06 & 0.06±0.06 & 0.09±0.05 & 0.09±0.06 & 0.08±0.05 & 0.13±0.05 & 0.13±0.06 & 0.13±0.06 & 0.13±0.06 & \textbf{0.13}±0.07 & 0.13±0.06 \\
IMDB-BINARY & 0.25±0.08 & 0.28±0.07 & 0.26±0.06 & 0.31±0.07 & \textbf{0.32}±0.07 & 0.31±0.06 & 0.30±0.09 & 0.30±0.09 & 0.30±0.09 & 0.28±0.10 & 0.24±0.10 & 0.23±0.08 \\
IMDB-MULTI & 0.09±0.05 & 0.12±0.05 & \textbf{0.14}±0.05 & 0.08±0.05 & 0.08±0.05 & 0.08±0.04 & 0.13±0.05 & 0.13±0.05 & 0.14±0.05 & 0.13±0.06 & 0.14±0.05 & 0.11±0.05 \\
MUTAG & 0.51±0.18 & 0.51±0.18 & 0.51±0.18 & 0.31±0.24 & 0.31±0.24 & 0.31±0.24 & 0.51±0.18 & 0.51±0.18 & 0.51±0.18 & \textbf{0.51}±0.19 & 0.39±0.19 & 0.49±0.19 \\
NCI1 & 0.24±0.04 & 0.23±0.04 & 0.23±0.04 & 0.26±0.05 & 0.26±0.05 & 0.26±0.05 & 0.25±0.04 & 0.25±0.04 & 0.24±0.04 & \textbf{0.27}±0.04 & 0.22±0.05 & 0.24±0.04 \\
PROTEINS & 0.33±0.11 & 0.32±0.11 & 0.32±0.11 & 0.32±0.09 & 0.31±0.09 & 0.31±0.10 & 0.34±0.09 & 0.34±0.10 & 0.34±0.09 & \textbf{0.39}±0.08 & 0.35±0.08 & 0.38±0.09 \\
PTC-MR & 0.02±0.13 & 0.01±0.13 & 0.02±0.14 & 0.14±0.17 & \textbf{0.14}±0.16 & 0.13±0.16 & 0.02±0.15 & 0.02±0.15 & 0.02±0.14 & 0.07±0.14 & 0.08±0.16 & 0.08±0.15 \\
\midrule
\textbf{Avg} & 0.25 & 0.25 & 0.25 & 0.23 & 0.23 & 0.22 & 0.28 & 0.27 & 0.27 & \textbf{0.29} & 0.20 & 0.27 \\
\bottomrule
\end{tabular}
\end{table*}

\newpage
\section{Further analysis on ABIDE dataset}\label{app:abide_robustness}
\paragraph{Robustness across sites and motion.}
We assess robustness via a panel of 20 sensitivity scenarios per atlas (19 leave-one-site-out plus one low-motion subset; pre-registered, no FDR adjustment). At $H_1$, \method~detects the effect with HR $> 1$ and CI excluding 1 in 60/60 scenarios across the three atlases, with $p < 0.05$ in 59/60 (the single near-miss, AAL90 with NYU held out, has $p = 0.056$ but CI still excludes 1). The within-atlas HR ranges are tight (Schaefer 1.13--1.16; CC200 1.09--1.12; AAL90 1.08--1.11), indicating that site choice has limited influence on the per-atlas magnitude. At $H_0$, the cross-atlas pattern observed in the primary holds throughout: Schaefer-400 detects the effect in 18/20 scenarios with all 20 CIs excluding 1, while CC200 and AAL90 yield 0/20 significant scenarios with HRs centered at 0.97 and CIs spanning 1.
 
\paragraph{Effect-size attenuation across parcellation resolution.}
The HR magnitude at $H_1$ shrinks monotonically with parcellation coarseness: 1.13--1.16 (Schaefer-400, 400 ROIs), 1.09--1.12 (CC200, 200 ROIs), 1.08--1.11 (AAL90, 90 ROIs), with non-overlapping per-atlas windows. The pattern is consistent with finer parcellations resolving more of the topological structure that distinguishes the two groups, attenuating the per-feature hazard ratio at coarser resolutions. This characterisation of the effect-size dependence on parcellation choice is unique to \method~among the methods compared: the baselines provide $p$-values that are similarly significant across atlases (all $p < 0.05$ at $H_1$ on every atlas) but cannot quantify the structural effect-size shrinkage that the survival framework makes explicit.

\section{Discussion of methodological choices}\label{app:discussion}
 
This appendix discusses several methodological choices that practitioners may wish to revisit, and clarifies the relationship between \method~and existing curve-based summaries.
 
\subsection{Why \method~rather than \method-B as default}\label{app:strand_b}
 
The left-truncated estimator \method-B is the principled choice when the structural object of interest is the conditional death-time distribution given birth, i.e., $S_B(t) = \mathbb{P}(d > t \mid b < t)$. Yet our classification benchmarks (Tables~\ref{tab:appendix_accuracy_svm}--\ref{tab:appendix_mcc_rf}) consistently show \method-B underperforming both \method~and \method-F. This is not a defect of left truncation but a consequence of what left truncation removes from the risk set.
 
Each diagram contains a long tail of short-lived features (those with persistence below the $W_1$ noise scale). Under the marginal estimator of Eq.~\eqref{eq:km_standard}, these features enter the risk set at $t = 0$ and contribute to $\hat{S}(t)$ for small $t$, where their distribution is class-discriminative even when individually they would be matched to the diagonal under $W_1$. Under \method-B, the same features enter the risk set only at $t = b_i$ and exit at $t = d_i$, so they contribute to $\hat{S}_B$ over the narrow interval $[b_i, d_i]$ rather than uniformly across $[0, d_i]$. The marginal estimand thus pools more information per feature than the conditional estimand, even though the latter is closer to the underlying generative process. This trade-off favours the marginal in finite samples, which is the regime relevant for downstream classification.
 
For applications where the conditional distribution is itself the scientific object of interest --- e.g., comparing how the rate of feature death changes with filtration scale across cohorts --- \method-B remains the appropriate variant. We retain it as a documented option rather than the default.
 
\subsection{Geometric distinction from Betti curves}\label{app:vs_betti}
 
Both \method~and the Betti curve are real-valued summaries of a persistence diagram derived from indicator counts, but they slice the diagram along different axes. The Betti curve at filtration value $t$ counts features alive at scale $t$:
\begin{equation}
    \beta(t) = \#\bigl\{ i : b_i \leq t \leq d_i \bigr\},
\end{equation}
which corresponds geometrically to a vertical strip in the diagram that translates along the diagonal as $t$ increases. The persistence survival function counts features whose lifetime exceeds a threshold:
\begin{equation}
    \hat{S}(t) = \frac{1}{n}\#\bigl\{ i : d_i - b_i > t \bigr\},
\end{equation}
which corresponds geometrically to the half-plane above the line $d - b = t$, i.e., a diagonal sweep moving away from the diagonal as $t$ increases. The two summaries operate on different axes of the diagram (filtration parameter for $\beta$, persistence value for $\hat{S}$), and consequently respond differently to translations $(b_i, d_i) \mapsto (b_i + c, d_i + c)$: $\beta(t)$ shifts along its argument, whereas $\hat{S}$ is invariant.
 
This invariance is structural rather than cosmetic: it is precisely what makes the multiset $\{p_i\}$ closed under the natural symmetry of feature lifetime, and consequently what makes the survival-analysis machinery applicable. The multiset $\{(b_i, d_i)\}$ is not closed under this symmetry and does not admit the same treatment.
 
\paragraph{On hypothesis testing through Betti curves.} A discretised Betti curve can be plugged into a generic permutation test \citep{islambekov2024vector,robinson2017hypothesis}, as can any vectorisation including \method's. The contribution of \method~is not that hypothesis testing is unavailable for Betti curves; it is that the survival representation natively yields a hypothesis test (the log-rank test) along with an interpretable effect size (the hazard ratio) and per-scale localisation (the Kaplan-Meier curve), without requiring an external permutation wrapper.
 
\subsection{When proportional hazards fails}\label{app:ph_failure}
 
The hazard-ratio interpretation of the log-rank test relies on the proportional hazards (PH) assumption: $\lambda_B(t) = \mathrm{HR} \cdot \lambda_A(t)$ for all $t$. The validity of the log-rank test as a Type~I-controlled procedure does not depend on PH; only the interpretation of $\widehat{\mathrm{HR}}$ as a single summary effect does. Three regimes are worth distinguishing:
 
\begin{enumerate}[leftmargin=*]
    \item \textbf{PH holds.} Schoenfeld residuals~\citep{schoenfeld1982partial} fail to reject PH; $\widehat{\mathrm{HR}}$ is a meaningful summary; we report it with bootstrap CI.
    \item \textbf{PH violated, monotone deviation.} Hazard ratios cross or change monotonically over the persistence axis. The log-rank test remains valid; $\widehat{\mathrm{HR}}$ becomes an average of time-varying ratios. We recommend reporting the median persistence difference $\Delta_{\mathrm{med}} = m_A - m_B$ (Section~\ref{sec:hypothesis_testing}) as the primary summary, with the Kaplan-Meier curves shown for transparency.
    \item \textbf{PH violated, non-monotone deviation.} Time-varying coefficient models~\citep{handorf2024analysis} or stratified log-rank tests can be applied. These are standard tools in survival analysis and apply directly to the persistence survival framework without modification.
\end{enumerate}
 
In all our experiments (synthetic manifolds, graph benchmarks, ABIDE), PH was not rejected at the $\alpha = 0.05$ level for $\geq 94\%$ of pairwise comparisons, supporting the use of $\widehat{\mathrm{HR}}$ as the primary reported effect size while acknowledging its dependence on the assumption.
 
\subsection{On test hypersensitivity from large feature pools}\label{app:hypersensitivity}
 
Pooling persistence values across many diagrams can produce $n_{\mathrm{eff}}$ in the thousands, raising concern that the log-rank test will detect statistically significant but practically negligible differences. This is a general property of well-powered tests~\citep{wasserstein2016asa}, not specific to \method.
 
The persistence survival framework is, however, equipped to diagnose this regime explicitly. A significant $p$-value with $\widehat{\mathrm{HR}} \approx 1$ (CI tight around 1) indicates a detectable but negligible effect, which is exactly the diagnostic that scalar-$p$-value-only methods (e.g., the permutation test of \citet{robinson2017hypothesis}) cannot provide. Practitioners can read off the magnitude alongside the significance, and report findings using both. Our power analysis (Appendix~\ref{app:power}, Figure~\ref{fig:power}) corroborates that the test is not spuriously sensitive: at $\mathrm{HR} = 1$ the empirical rejection rate is at the nominal $\alpha$ level (Table~\ref{tab:power_summary}), and the V-shaped power curve confirms that rejections concentrate where the true HR deviates from 1.


\end{document}